\documentclass[
]{ceurart}

\usepackage{amsmath}
\usepackage{graphicx}
\usepackage{microtype}
\usepackage{dsfont}
\usepackage{stmaryrd}
\usepackage{tikz} 
\usepackage{comment}
\usetikzlibrary{shapes.misc,automata,positioning,arrows}
\usepgflibrary{shapes.arrows} 
\usepackage{multicol}
\usepackage{mathrsfs}
\usepackage{xr}

\usepackage{algorithm}
\usepackage[noend]{algorithmic}

\usepackage{listings}
\usepackage{amsthm}
\newtheorem{theorem}{Theorem}
\newtheorem{lemma}{Lemma}
\newtheorem{definition}{Definition}
\newtheorem{example}{Example}
\newtheorem{proposition}{Proposition}
\newtheorem{corollary}{Corollary}

\begin{document}


\newcommand{\Lang}{\mathcal{L}}
\newcommand{\KLMLang}{\Lang^{\twiddle}}
\newcommand{\DRSLLang}{\smash{\KLMLang_{\mathbb{S}}}}
\newcommand{\negKLMLang}{\Lang^{\ntwiddle}}
\newcommand{\negDRSLLang}{\negKLMLang_\mathbb{S}}


\newcommand{\A}{\mathcal{A}}
\newcommand{\As}{\A_s}
\newcommand{\Aspsi}{\A^\psi_s}
\newcommand{\Asipsi}{\A^\psi_{s,\omega}}

\newcommand{\I}{\mathcal{I}}
\newcommand{\Is}{\I_s}
\newcommand{\Ispsi}{\I^\psi_s}
\newcommand{\Isipsi}{\I^\psi_{s,\omega}}

\newcommand{\KB}{\mathcal{K}}


\newcommand{\dbox}{\mathrlap{\sim}\square}
\newcommand{\ddiamond}{\mathrlap{\tiny\sim}\Diamond}

\newcommand{\dentails}{\mid\hskip-0.40ex\approx}
\newcommand{\ndentails}{\not\mid\hskip-0.40ex\approx}
\renewcommand{\bar}[1]{\overline{#1}}
\newcommand{\twiddle}{\mathrel|\joinrel\sim}
\newcommand{\ntwiddle}{\mathrel|\joinrel\not\sim}
\newcommand{\usually}{\:\raisebox{0.45ex}{\ensuremath{\sqsubset}}\hskip-1.7ex\raisebox{-0.6ex}{\scalebox{0.9}{\ensuremath{\sim}}}\:}
\newcommand{\nusually}{\not\hspace{-1.1mm}\usually}
\newcommand{\St}{\ensuremath{\text{\it S}}} 
\newcommand{\lab}{\ensuremath{\ell}} 
\newcommand{\pref}{\prec} 
\newcommand{\npref}{\not\pref}
\newcommand{\PM}{\mathscr{P}} 
\newcommand{\RM}{\mathscr{R}} 
\newcommand{\CM}{\mathscr{C}} 
\newcommand{\RSet}{\ensuremath{\mathcal{R}}} 
\newcommand{\Rentails}{\Centails{\RSet}} 
\newcommand{\nRentails}{\not\Rentails}
\newcommand{\Sat}{\mathrel|\joinrel\equiv} 
\newcommand{\nSat}{\not\Sat}
\newcommand{\states}[1]{\llbracket#1\rrbracket} 
\newcommand{\rk}{\ensuremath{\text{\it rk}}} 
\newcommand{\nme}{\ensuremath{\mathscr{R}}} 
\newcommand{\RC}{\ensuremath{\mathcal{R}^{\twiddle}}} 
\newcommand{\rel}{\ensuremath{\text{\it R}}} 
\newcommand{\sel}{{\ensuremath{\text{\sf sel}}}} 

\newcommand{\OI}{\ensuremath{\mathcal{O}}} 
\newcommand{\PI}{\ensuremath{\mathcal{P}}} 
\newcommand{\RI}{\ensuremath{\mathcal{R}}} 
\newcommand{\dsrel}{\leadsto}
\newcommand{\ndsrel}{\not\dsrel}
\newcommand{\D}{\ensuremath{\mathcal{D}}} 
\newcommand{\K}{\widetilde{\mathbb{K}}}
\newcommand{\Rs}{\mathscr{R}} 
\newcommand{\X}{\mathcal{X}} 
\newcommand{\generalent}{\entails_{\X}} 

\newcommand*{\myleftmid}{%
	\mathrel{\vcenter{\offinterlineskip
			\vskip-0.25ex\hbox{$\shortmid$}}}}
\newcommand*{\myrightmid}{%
	\mathrel{\vcenter{\offinterlineskip
			\vskip-0.7ex\hbox{$\shortmid$}}}}
\newcommand*{\twosim}{%
	\mathrel{\vcenter{\offinterlineskip
			\vskip0.05ex\hbox{$\sim$}\vskip0.25ex\hbox{$\sim$}}}}
\newcommand{\bartwosim}{\mathrel{\myleftmid}\hskip-0.03ex\joinrel\twosim}
\newcommand{\dnec}{\mathrel{\bartwosim}\hskip-.05ex\joinrel\myrightmid}
\newcommand{\dposs}{\scalebox{0.8}{\raisebox{-0.2ex}{\rotatebox{57}{\ensuremath{\dnec}}}}\hskip-0.3ex}
\newcommand{\smallsim}{\scalebox{0.86}{\raisebox{0.4ex}{\ensuremath{\sim}}}}
\newcommand{\dobl}{\mathrel{\obl}\hskip-1.9ex\joinrel\smallsim\hskip0.2ex}
\newcommand{\Langd}{\widetilde{\Lang}}
\newcommand{\smallsimforall}{\scalebox{0.68}{\raisebox{0.76ex}{\ensuremath{\sim}}}}
\newcommand{\dforall}{\forall\hskip-1.06ex\raisebox{0.7ex}{\scalebox{0.7}{\ensuremath{\textcolor{white}{\bullet}}}}\hskip-1.88ex\joinrel\smallsimforall}
\newcommand{\smallsimexists}{\scalebox{0.65}{\raisebox{0.62ex}{\ensuremath{\sim}}}}
\newcommand{\dexists}{\exists\hskip-1.5ex\scalebox{1.2}{\ensuremath{\textcolor{white}{\bullet}}}\hskip-2.05ex\joinrel\smallsimexists}

\newcommand{\ratent}{\vdash_{r}}

\newcommand{\minent}{\entails_R^{\cup}}
\newcommand{\ratentabox}{\vdash_{r}}
\newcommand{\tw}[1]{\widetilde{#1}}
\newcommand{\E}{\ensuremath{\mathcal{E}}} 
\newcommand{\e}{\ensuremath{\varepsilon}} 

\newcommand{\darrow}{\mathrel\sim\joinrel\rightarrow}

\setlength{\parindent}{0pt}

\newcommand{\Nick}[1]{\textcolor{blue}{#1}}

\copyrightyear{2026}
\copyrightclause{Copyright for this paper by its authors.
  Use permitted under Creative Commons License Attribution 4.0
  International (CC BY 4.0).}

\conference{ }

\title{Standpoint Logics with Defeasible Beliefs}

\author[1]{Nicholas Leisegang}
\cormark[1]
\address[1]{University of Cape Town and CAIR, South Africa}
\author[1]{Thomas Meyer}
\author[2,3]{Sebastian Rudolph}
\address[2]{Technische Universität Dresden, Dresden, Germany}
\address[3]{ScaDS.AI -- Center for Scalable Data Analytics and Artificial Intelligence Dresden/Leipzig, Germany}

\cortext[1]{Corresponding author.}

\begin{abstract}
    In this paper, we integrate the defeasible logic of Kraus, Lehmann and Magidor (KLM) with the standpoint logic framework of Gómez Álvarez and Rudolph. This is done with the goal of formally expressing knowledge taking into account multiple (possibly contradicting) viewpoints, which in turn may hold defeasible beliefs. In doing so, we utilise Defeasible Restricted Standpoint Logics (DRSL), introduced by Leisegang et al. Our work expands on previous work by providing a foundational representation result for DRSL semantics and systematically lifting several well-known entailment relations from the propositional case to the standpoint-enhanced setting. In particular, we characterise the semantics for DRSL through a set of KLM-style postulates adapted for the standpoints case. We furthermore provide a means to lift preferential entailment, and the class of entailment relations based on single ranking functions from the purely propositional to the standpoint-enhanced context, including rational and lexicographic closure. We show this can be done equivalently through semantic and algorithmic means. Furthermore, we show that, for each considered form of entailment, the complexity class of entailment checking does not change when moving from propositional KLM to DRSL.
\end{abstract}


\maketitle

\section{Introduction}

\textit{Standpoint logics} denote a modal logic framework introduced by Rudolph and Gómez Álvarez \cite{alvarezrudolph:propositionalstdpt} with the purpose of expressing multiple standpoints or perspectives, which may hold conflicting beliefs about some topic, while still maintaining consistency in the logic. Much of the work in standpoint logics concerns the addition of \textit{standpoint modalities} to base logics such as description logics \cite{alvarez:stdptlogicfocase,alvarezrudolphstrass:standpointEL}, while showing that the added expressivity does not increase the complexity of reasoning. Recent work has integrated standpoint modalities in non-monotonic formalisms, modelling situations where standpoints may encompass beliefs that resemble default rules or defeasible implications \cite{LMR2024,GS2026}. As a motivating case, consider the following example.

\begin{example}[\cite{LMR2024}]
    \label{example:originaltomatoes}
    In the 19th century, crops imported into the USA were divided into fruits and vegetables, where fruits were exempt from import tax. This led to a court case on whether a tomato should be legally classified as a fruit or a vegetable. From a botanical standpoint, tomatoes are fruits and all fruits are also vegetables. This can be expressed using standpoint logics with the formulas:
    $\Box_{B}(tomato \rightarrow fruit)$ and $ \Box_{B}(fruit \rightarrow veg.)$,
    where $B$ represents the botanical standpoint. However, the court considered a different standpoint based on the culinary use of tomatoes, according to which vegetables are  those crops which are savoury and fruits are those which are sweet, giving us the beliefs: 
    $\Box_{C}(savoury \leftrightarrow veg.)$, $\Box_C(sweet \leftrightarrow fruit)$, $\Box_{C}(tomato\twiddle savoury)$, and $ \Box_{C}((fruit\twiddle \neg veg.)\wedge (veg.\twiddle \neg fruit))$, 
    where $C$ represents the ``culinary'' standpoint. Here, the third statement tells us that tomatoes are usually used in savoury dishes, and the last proposition states that fruits and vegetables are usually considered distinct from each other. Ultimately, courts agreed with the culinary standpoint that tomatoes were usually considered vegetables. We can represent this by $L\preceq C$, $\Box_{L}(veg.\rightarrow \neg fruit)$, where $L\preceq C$ tells us that the legal standpoint $L$ holds true each conclusion of $C$'s standpoint, and the second proposition states that fruits and vegetables are strictly distinct, legally speaking. This system allows for both internal exceptions, and strict disagreements between standpoints: $B$ believes that every fruit is a vegetable, while $L$ believes that no fruit is a vegetable. We also see that it is possible that a certain kind of exceptional tomato is not savoury from a culinary standpoint. In fact, from a culinary perspective it is possible to consider a tomato both a fruit and vegetable, due to its exceptional nature (i.e. $\Diamond_C(tomato\rightarrow(fruit \wedge veg.)$).
    
\end{example}

This shows us a scenario in which there are both different standpoints with conflicting points of view, as well as standpoints that hold beliefs which are defeasible in their nature. For example, the rule $fruit\twiddle\neg vegetable$ allows for non-prototypical fruits to also be considered vegetables. The addition of standpoint modalities therefore adds a new layer of expressivity to the propositional defeasible reasoning of Kraus, Lehmann and Magidor (KLM) \cite{kraus:nonmonotonic}, in which multiple perspectives can be considered simultaneously. On the other hand, adding defeasibility to the beliefs held by standpoints allows for beliefs to be weakened, and therefore for more agreements in beliefs to be held across standpoints, ultimately increasing the inferences that can be drawn from the classical case. In our paper, we extend the work of Leisegang et al. \cite{LMR2024}, who introduce the language of Defeasible Restricted Standpoint Logic (DRSL), and the corresponding semantics of \textit{preferential standpoint structures}. Our contribution characterises preferential standpoint structures by defining a set of KLM-style postulates, and providing a representation result that each set of DRSL statements closed under these postulates can be represented by a preferential standpoint structure. We then consider entailment for DRSL, where we first lift monotonic preferential entailment from the propositional case to the standpoints case, as well as providing a method for lifting the class of non-monotonic entailment relations which are based on single ranking functions to the standpoint logic case. We further show that entailment-checking for DRSL remains in the same complexity class as in the propositional case, in both the monotonic and non-monotonic case. Our paper is structured as follows: Section \ref{section:prelims} introduces preliminary material from propositional KLM-style defeasible reasoning. Section \ref{section:preferential-DRSL} introduces the syntax and semantics for DRSL, and provides the representation result which shows soundness and completeness for our adapted set of postulates. Section \ref{section:entailments} establishes results for both monotonic and non-monotonic entailments for DRSL. Section \ref{section:relatedwork} considers related work and provides concluding remarks.

\section{Preliminaries}\label{section:prelims}
In this section, we introduce some basic results and definitions from propositional defeasible reasoning. The framework for defeasible reasoning we use was proposed by Kraus, Lehmann and Magidor (KLM) \cite{kraus:nonmonotonic}, who introduced an additional defeasible implication ``$\twiddle$'' to classical propositional logic, where $\alpha\twiddle\beta$ reads as \textit{``$\alpha$ typically implies $\beta$''}.

\begin{definition}\label{definition:KLM-propositional-language}
    The language of \textit{KLM propositional logic} $\mathcal{L}^{\twiddle}$ over a set of propositional atoms $\mathcal{P}$ is given by
    \[\phi::=\alpha\twiddle\beta\mid\phi\wedge \phi\ ,\]
    where $\alpha$ and $\beta$ are Boolean formulas with atoms in $\mathcal{P}$.
\end{definition}

This definition is a slight extension of the original language of KLM, which does not allow conjunction in the language. This distinction is important in our case when we apply modal operators across conjunctions of KLM implications, rather than single implications. It is also worth noting that the $\twiddle$ sign was originally introduced as a defeasible consequence operator on the meta level, whereas we use it as a form of defeasible implication within the language itself. Lastly, we recall that any Boolean formula $\alpha$ is semantically equivalent to $\neg \alpha\twiddle \bot$ \cite{casini:beyondratclosure}, and so our language can express classical Boolean statements, as well as defeasible implications. Abusing notation, we use the Boolean formula $\alpha$ as a shorthand for $\neg \alpha\twiddle \bot$. The semantics for $\KLMLang$ is defined below.

\begin{definition} \label{definition:propositional-preferential-interpretation}
   \cite{kraus:nonmonotonic} A \textit{preferential interpretation} over a set of propositional atoms $\mathcal{P}$ is a triple $\I=(W,l,<)$ where $W$ is a (possibly empty) set of states, $l:W\rightarrow 2^{\mathcal{P}}$ is a mapping from the set of states to the set of classical valuations on $\mathcal{P}$, and $<$ is a strict partial order on $W$ such that \textit{for every Boolean formula $\alpha$, the set $\llbracket \alpha \rrbracket:=\{w\in W\mid l(w)\Vdash \alpha\}$  has a minimal element with respect to $<$.}
\end{definition}

    The satisfaction relation $\Vdash$ for a preferential interpretation is defined as follows for Boolean formulas $\alpha$ and $\beta$: $\I\Vdash \alpha\twiddle\beta$ if $l(w)\Vdash \beta$ for all $w\in min_< \llbracket\alpha\rrbracket$; and $\I\Vdash \phi_1\wedge \phi_2$ if $\I\Vdash \phi_1$ and $\I\Vdash \phi_2$. A formula, $\phi\in\KLMLang$ is satisfiable if there exists some preferential interpretation $\I$ with $\I\Vdash \phi$, and this is extended to sets of formulas in the usual way. However, it is trivially the case that every set $A\subseteq\KLMLang$ is satisfiable, since the preferential model $\I$ where the set of states is empty satisfies all formulas in $\KLMLang$. For our purposes, we want to consider a stronger notion of satisfiability, where we only consider non-empty preferential models. We say a set $\A\subseteq\KLMLang$ is \textit{non-trivially satisfiable} if there exists a preferential interpretation $\I=(W,l,<)$ such that $\I\Vdash \phi$ for all $\phi\in\A$ and $W\neq \emptyset$. A foundational result of preferential semantics is their relation to the \emph{KLM postulates}. These are a set of properties, given in Figure \ref{fig:Prop-KLM-postulates} which describe the basic behaviour of the ``$\twiddle$'' connective. The result below allows to establish that any set of KLM-style defeasible implications is closed under the KLM postulates iff it can be represented by a single preferential interpretation.

\begin{figure}
  \small  \centering 
  \[(\textbf{Ref})\ \frac{}{\alpha\twiddle\alpha}\ \ \ \ \ \ (\textbf{LLE})\ \frac{\alpha\equiv\beta, \alpha\twiddle\gamma}{\beta\twiddle \gamma}\ \ (\textbf{RW})\ \frac{\alpha\vDash\beta,\gamma\twiddle\alpha }
    {\gamma\twiddle \beta}\] \[ (\textbf{And})\ \frac{\alpha\twiddle\beta,\alpha\twiddle\gamma}{\alpha\twiddle \beta\wedge \gamma}\  \ \ \ 
(\textbf{Or})\ \frac{\alpha\twiddle\gamma,\beta\twiddle\gamma}{\alpha\vee \beta\twiddle\gamma} \
  \ (\textbf{CM})\ \frac{\alpha\twiddle\beta, \alpha\twiddle\gamma}
    {\alpha\wedge \beta\twiddle\gamma}\ \ \] 
    \caption{KLM Postulates}
    \label{fig:Prop-KLM-postulates}
\end{figure}

\begin{theorem}\cite{kraus:nonmonotonic}\label{theorem:KLM-original-pref-rep-result}
    If a set $\mathcal{A}\subseteq \mathcal{L}^{\twiddle}$ is closed under the KLM postulates, then there is some preferential interpretation $\mathcal{I}$ such that $\phi\in \mathcal{A}$ iff $\mathcal{I}\Vdash \phi$.
\end{theorem}

We work under the assumption that defeasibility is added on top of classical propositional logic, and that any set of formulas which is considered closed under the KLM postulates is also closed under classical propositional deduction rules.


\section{Defeasible Beliefs in Standpoint Logics}{\label{section:preferential-DRSL}}

Standpoint logics are a family of modal logics used to describe scenarios involving different viewpoints which may hold conflicting beliefs. For each such standpoint $s$, one introduces modal operators $\Box_s$ and $\Diamond_s$. Then, $\Box_s\phi$ and $\Diamond_s\phi$  read as \textit{``it is unequivocal to $s$ that $\phi$''} and \textit{``it is possible to $s$ that $\phi$''}, respectively. In this section, we consider an extension of propositional standpoint logic, as defined by Gòmez Álvarez and Rudolph \cite{alvarezrudolph:propositionalstdpt}. In particular, we apply standpoint modalities to KLM-style defeasible implications by introducing a logic with formulas of the form $\Box_s(\alpha\twiddle \beta)$ or $\Diamond_s(\alpha\twiddle \beta)$. This represents the case where a standpoint holds a defeasible belief, and these sentences should be read as \textit{``it is unequivocal to $s$ that $\alpha$ usually implies $\beta$''} and \textit{``it is possible to $s$ that $\alpha$ usually implies $\beta$''}. In this section, we describe the syntax and semantics for DRSL, and show that this can be characterised through an extended system of postulates inspired by the original KLM postulates. We restrict the syntax to only allow for conjunction on the outermost level. This restriction is made with the spirit of propositional KLM in mind, in which disjunctions and negations of defeasible statements are not expressible. Such a restriction is also needed to facilitate the representation and complexity results found in the remainder of the paper.

\begin{definition} \cite{LMR2024}\label{definition:DRSL-language}
   A \textit{vocabulary} is a pair $\mathcal{V}=(\mathcal{P},\mathcal{S})$ where $\mathcal{P}$ is a finite set of propositional atoms and $\mathcal{S}$ is a finite set of standpoint symbols. The language of  \textit{Defeasible Restricted Standpoint Logic (DRSL)} $\mathcal{L}^{\twiddle}_\mathbb{S}$ is defined as follows,
            \[\psi::=\phi\mid \Box_s \psi\mid \Diamond_s \psi \mid \psi\wedge \psi\text{ or }\psi::=s\preceq t\]
			where $s,t\in \mathcal{S}\cup \{*\}$ and $\phi\in\mathcal{L}^{\twiddle}$ with atoms in $\mathcal{P}$.
\end{definition}

Statements of the form $s\preceq t$ are referred to as \textit{standpoint sharpening} statements, and intuitively, they say that $s$ is a more specific version of $t$'s viewpoint. The $*$ symbol is referred to as the \textit{universal standpoint} and represents the beliefs which all standpoints agree upon. The semantics for this are given by \textit{preferential standpoint structures}, which again are built on top of preferential interpretations given in propositional defeasible reasoning. 

\begin{definition}\cite{LMR2024} \label{definition:preferential-standpoint-structure}
    A \textit{preferential standpoint structure} is a triple $M=(\Pi,\sigma,\tau)$, where:
    \begin{enumerate}
        \item $\Pi$ is a non-empty set of \emph{precisifications} (possible worlds).
        \item $\sigma:\mathcal{S}\cup\{*\}\rightarrow2^\Pi$ is a map which assigns a non-empty set of precisifications to each standpoint, such that $\sigma(*)=\Pi$.
        \item  $\tau:\Pi\rightarrow \mathbb{I}$ is function where $\mathbb{I}$ is the set of preferential interpretations over the set $\mathcal{P}$.  That is, $\tau$ is a map which assigns to each precisification a preferential interpretation.
    \end{enumerate}
\end{definition}

	\begin{definition}\cite{LMR2024} \label{definition:satisfaction-preferential-standpoint-structure}
			Given a preferential standpoint structure $M$ and a precisification $\pi\in \Pi$, the satisfaction relation $\Vdash$ is defined as followed (where $\phi\in\mathcal{L}^{\twiddle}$, $\psi\in\mathcal{L}^{\twiddle}_{\mathds{S}}$ and $s,t\in\mathcal{S}\cup \{*\}$):
			\begin{itemize}
				\item $M, \pi \Vdash \phi$ iff $\tau(\pi)\Vdash \phi$, 
				\item $M, \pi \Vdash \Box_s \psi$ iff $M,\pi' \Vdash \psi$ for all $\pi'\in \sigma(s)$,
				\item $M, \pi \Vdash \Diamond_s \psi$ iff $M,\pi' \Vdash \psi$ for some $\pi'\in \sigma(s)$,
                \item $M,\pi\Vdash \psi_1\wedge \psi_2$ iff $M,\pi\Vdash \psi_1$ and $M,\pi\Vdash \psi_2$,
                \item $M,\pi\Vdash s\preceq t$ iff $\sigma(s)\subseteq \sigma (t)$,
                \item $M\Vdash \psi$ iff $M,\pi\Vdash \psi$ for all $\pi\in\Pi$.
			\end{itemize} 

 \end{definition}


Intuitively, $\sigma$ assigns to each standpoint $s$ a set of reasonable ways to understand $s$'s set of beliefs, where each $\pi\in \sigma(s)$ denotes one specific fixed way to understand $s$'s beliefs, via the preferential interpretation $\tau(\pi)$. These are closely derived from the semantics for classical propositional standpoint logic, where each $\tau(\pi)$ is a classical valuation, rather than a preferential interpretation \cite{alvarezrudolph:propositionalstdpt}. 

\begin{definition}\label{definition:satisfiable-}
    A set of DRSL formulas $\A\subseteq \DRSLLang$ is \textit{satisfiable} if there exists some preferential standpoint structure $M=(\Pi,\sigma,\tau)$ in which, for all $\pi\in \Pi$ the set of states in $\tau(\pi)$ is non-empty, and such that 
    $M\Vdash \phi$ for all $\phi\in\A$, 
\end{definition}

This is stronger than satisfiability in the propositional KLM setting, and is closer to non-trivial satisfiability for propositional KLM. In particular, every set of defeasible implications in $\KLMLang$ is satisfiable, while not every set in $\DRSLLang$ is satisfiable.

\begin{example}\label{example:tomato-model}
    Consider the ``tomato'' knowledge base $\KB_{T}$ from Example \ref{example:originaltomatoes}, with the propositional atoms shortened:

\begin{center}
    \small$\KB_{T}=\{\Box_B(t\rightarrow v),\Box_B(f\rightarrow v), \Box_C(v\leftrightarrow s_a),   \Box_C(f\leftrightarrow s_w),$

    $\Box_C(t\twiddle s_a), \Box_{C}((f\twiddle \neg v)\wedge (v\twiddle \neg f)), L\preceq C, \Box_L(v\rightarrow\neg f) \}$
        
\end{center}

\normalsize We exhibit a preferential standpoint structure $M=(\Pi,\sigma,\tau)$ which satisfies $\KB_T$ as follows: Let $\Pi=\{\pi_1,\pi_2,\pi_3\}$ where $\tau(\pi_1)$ and $\tau(\pi_2)$ are preferential interpretations each with a single state which have underlying valuations of $\{t,f,v\}$ and $\{t,v,s_a\}$ respectively. Let $\tau(\pi_3)=(W,l,<)$ where $W$ has 3 states $s_1$, $s_2$ and $s_3$; ${<}=\{(s_1,s_3),(s_2,s_3)\}$; and $l(s_1)=\{f,s_w\}$ $l(s_2)=\{v,s_a\}$, $l(s_3)=\{t,f,s_w,v,s_a\}$. Then defining $\sigma$ by $\sigma(B)=\{\pi_1\}$, $\sigma(L)=\{\pi_2\}$ and $\sigma(C)=\{\pi_2,\pi_3\}$ we see that $M$ satisfies $\KB_T$.
\end{example}


We now describe a set of proof-theoretic postulates which characterise preferential semantics for DRSL. As with KLM and propositional standpoint logic, we operate under the assumption that our logic is built on top of Boolean propositional logic. We therefore accept that any axioms or rules of inference which hold in Boolean propositional logic hold in our logic. Besides this, the postulates we give here can be broadly split into two groups. The first of these is an adapted set of the modal axioms which characterise standpoint modalities. In the classical case standpoint logic acts as a multimodal variant of \textbf{KD45} with additional axioms tailored for standpoint sharpenings. $\Box_*$ acts as an \textbf{S5} modality, and therefore has additional properties \cite{alvarezrudolph:propositionalstdpt}. In our setting, we include an adapted version of the original axioms for standpoint modalities, which are sufficient to cover the restrictions in our syntax. These axioms rephrased in terms of Gentzen-style rules in order to be expressible in our restricted version of standpoint logic, which does not allow for unrestricted material implication or disjunction. Furthermore, the restriction of negation means $\Box_s$ and $\Diamond_s$ must be treated separately, and not simply as duals of each other; although semantically they behave as duals. The derived modality postulate for DRSL are given in Figure~\ref{fig:DRSL-modal-rules}. They are named directly after the original classical standpoint logic axioms from which they are derived \cite{alvarezrudolph:propositionalstdpt}, with some of the original modal axioms requiring two or more rules in our setting of restricted expressivity (for example, \textbf{P.a} and \textbf{P.b}). Additional postulates describe how modalities distribute over conjuncts, and govern the behaviour of standpoint sharpenings. Besides these modal rules which govern the external structure between precisifications in the semantics, we also require rules which describe how the preferential interpretations underlying the precisifications behave. In order to describe this, we adapt the original KLM postulates, as given by Kraus et al.~\cite{kraus:nonmonotonic}, into the standpoint logic modal case. These are given in Figure \ref{fig:DRSL-KLM-postulates}. From each original KLM postulate, we derive a pair of postulates which correspond to preferential reasoning which is both global or local to a given standpoint. Each ``a.'' postulate describes how the KLM postulates apply when defeasible beliefs hold throughout an entire standpoint, while each ``b.'' postulate describes conclusions reached in a specific possible precisification for a standpoint. That is, in each rule we cumulatively add a new conjunct to the premise, ultimately obtaining a large $\Diamond$-bound conjunction of defeasible implications. This is done in order to collect every conclusion valid at a possible precisification, and make sure that we do not lose any conclusions which may be derived relative to the same possible world.

\begin{figure}
   \small \centering 
  \[(\textbf{RN})\ \frac{\phi}{\Box_s\phi}\ \ \ \ \ \  (\textbf{K.a.})\frac{\Box_s(\phi\rightarrow \psi),\Box_s \phi}{\Box_s\psi} \ \ \ (\textbf{K.b.})\ \frac{\Box_s\phi, \Diamond_s\psi}{\Diamond_s(\phi\wedge \psi)}\ \ \ \
  \ (\textbf{K.c.})\frac{\Diamond_s((\phi\rightarrow \psi)\wedge \phi \wedge \Gamma)}{\Diamond_s(\psi\wedge (\phi\rightarrow \psi)\wedge \phi \wedge \Gamma)} \] \[ 
  (\textbf{5'.})\ \frac{\Diamond_t\Box_s\phi}{\Box_s\phi} \ \ \ \ \ (\textbf{T}^*)\frac{\Box_*\phi}{\phi}\ \ \ (\textbf{P.a.})\frac{s\preceq t, \Box_t \phi}{\Box_s\phi} \ \ \ \ (\textbf{P.b.})\ \frac{s\preceq t, \Diamond_s \phi}{\Diamond_t\phi} \ \ \ \ \ \ (\textbf{D.})\ \frac{\Box_s \phi}{\Diamond_s \phi} \ \ \ \ \ \ \ (\textbf{4'.}) \frac{\Diamond_t\Diamond_s\phi}{\Diamond_s\phi} \] \[ \ (\Box\textbf{-Dist.a.})\ \frac{\Box_s(\phi\wedge \psi)}{\Box_s\phi\wedge \Box_s\psi}\ \ \ \ \ \ (\Box\textbf{-Dist.b.})\ \frac{\Box_s\phi\wedge \Box_s\psi}{\Box_s(\phi\wedge \psi)} \ \ \ (\Diamond\textbf{-Dist.})\ \frac{\Diamond_s(\phi\wedge \psi)}{\Diamond_s\phi\wedge \Diamond_s\psi} \] \[
(\preceq\textbf{-Refl.})\frac{}{s\preceq s}\ \ \ \ (*\textbf{-Top})\frac{}{s\preceq *} \ \ \ \ \ (\preceq\textbf{-Trans.})\frac{s\preceq t, t\preceq u}{s\preceq u} \]
    \caption{Modality Postulates for DRSL}
    \label{fig:DRSL-modal-rules}
\end{figure}

\begin{definition}
    A set $\mathcal{A}\subseteq \mathcal{L}^{\twiddle}_\mathbb{S}$ is \textit{preferentially closed} if it is closed under classical propositional logic and the rules in Figures \ref{fig:DRSL-modal-rules} and \ref{fig:DRSL-KLM-postulates}. That is, if the premises of any of the rules occur in $\mathcal{A}$, then the consequences of the rule are in $\mathcal{A}$. 
\end{definition}

We show that these postulates accurately characterise the semantics of DRSL. We do this by showing our semantics are sound and complete with respect to the rules in Figures \ref{fig:DRSL-modal-rules} and \ref{fig:DRSL-KLM-postulates}. Moreover, we show that any set of DRSL formulas closed under preferential reasoning can be represented by a single preferential standpoint structure. We begin with soundness:

\begin{lemma}\label{theorem:main-soundness}
    Any preferential standpoint structure satisfies the rules given in Figures \ref{fig:DRSL-modal-rules} and \ref{fig:DRSL-KLM-postulates}, and satisfies classical propositional logic.
\end{lemma}

We now show the completeness result that says any set of DRSL sentences which is preferentially closed can be characterised by some preferential standpoint structure. In order to show this theorem holds, we utilise the original representation result given in Theorem \ref{theorem:KLM-original-pref-rep-result} for preferential consequence relations in the propositional case. We introduce a lemma which allows us to only consider DRSL formulas in some normal form. That is, we show that any DRSL formula (which is not a standpoint sharpening) can be reduced to an equivalent formula in this normal form, and the reduction of a formula to normal form can be done through applying the postulates and through semantic means.

 \begin{figure}
   \small \centering 
  \[(\textbf{Ref.})\ \frac{}{\alpha\twiddle\alpha}\ \ \ \ \ \ (\textbf{LLE.a.})\ \frac{\Box_s(\alpha\leftrightarrow\beta), \Box_{s}(\alpha\twiddle\gamma)}{\Box_s(\beta\twiddle \gamma)}\ \ \ \ 
(\textbf{LLE.b.})\ \frac{\Diamond_{s}((\alpha\leftrightarrow\beta)\wedge (\alpha\twiddle\gamma)\wedge \Gamma)}{\Diamond_{s}((\beta\twiddle \gamma)\wedge \Gamma')} \]
\hspace{2pt}
  \[ (\textbf{RW.a.})\ \frac{\Box_s(\alpha\rightarrow\beta),\Box_{s}(\gamma\twiddle\alpha) }
    {\Box_s(\gamma\twiddle \beta)}\ \ \ \ \ \ 
(\textbf{RW.b})\ \frac{\Diamond_{s}((\alpha\rightarrow\beta)\wedge (\gamma\twiddle\alpha)\wedge \Gamma)}
    {\Diamond_s((\gamma\twiddle \beta)\wedge \Gamma')} \]
  \hspace{2pt}
  \[ (\textbf{And.a.})\ \frac{\Box_s(\alpha\twiddle\beta),\Box_{s}(\alpha\twiddle\gamma)}{\Box_s(\alpha\twiddle (\beta\wedge \gamma))}\ \ \ \ \ \ 
(\textbf{And.b.})\ \frac{\Diamond_{s}((\alpha\twiddle\beta)\wedge (\alpha\twiddle\gamma)\wedge \Gamma)}{\Diamond_s((\alpha\twiddle (\beta\wedge \gamma))\wedge  \Gamma')} \]
\hspace{2pt}
  \[ (\textbf{Or.a.})\ \frac{\Box_s(\alpha\twiddle\gamma),\Box_{s}(\beta\twiddle\gamma)}
{\Box_s((\alpha\vee \beta)\twiddle\gamma)}\ \ \ \ \ \ 
(\textbf{Or.b.})\ \frac{\Diamond_{s}((\alpha\twiddle\gamma)\wedge (\beta\twiddle\gamma)\wedge \Gamma)}{\Diamond_s(((\alpha\vee \beta)\twiddle\gamma)\wedge \Gamma')} \]
  \[ (\textbf{CM.a.})\ \frac{\Box_s(\alpha\twiddle\beta), \Box_{s}(\alpha\twiddle\gamma)}
    {\Box_s((\alpha\wedge \beta)\twiddle\gamma)}\ \ \ \ \ \ 
(\textbf{CM.b.})\ \frac{\Diamond_{s}((\alpha\twiddle\beta)\wedge (\alpha\twiddle\gamma)\wedge \Gamma)}{\Diamond_s(((\alpha\wedge \beta)\twiddle\gamma)\wedge \Gamma')}\]

    \caption{KLM Postulates for DRSL. Here, $\Gamma'$ is used as shorthand to denote the original diamond-bound sentence occurring in the premise of each rule.}
    \label{fig:DRSL-KLM-postulates}
\end{figure}

\begin{definition}\cite{LMR2024}
    A DRSL formula $\phi$ is in \textit{normal form} if it is in the form
    \[\phi=\bigwedge^n_{i=1}\phi_i\]
    where for each $i\in\{1,...,n\}$ either (a) $\phi_i\in \mathcal{L}^{\twiddle}$, (b) $\phi_i=\Box_s\psi$, or (c) $\phi_i=\Diamond_s\psi$ for some  $\psi\in \mathcal{L}^{\twiddle}$ and $s\in\mathcal{S}$. That is, $\phi$ is a conjunction of propositional KLM formulas bound by at most one standpoint modality.
\end{definition}

 Leisegang et al.~\cite{LMR2024} show that each DRSL formula has a semantically equivalent formula in normal form.

\begin{lemma}\label{corollary:semantic-normal-form}
    For any $\phi\in \DRSLLang$ which is not a standpoint sharpening, there exists a formula $\phi'\in\DRSLLang$ in normal form such that, for any preferential standpoint structure $M$, we have $M\Vdash \phi$ iff $M\Vdash \phi'$.
\end{lemma}

The following result shows that a formula in normal form can be generated via the rules in Figures \ref{fig:DRSL-modal-rules} and \ref{fig:DRSL-KLM-postulates}. Furthermore, since we have shown soundness for each of our postulates, it follows that the formula in normal form obtained through applying the postulates is a formula which is equivalent in the semantics.

\begin{lemma}\label{lemma:proof-theoretic-normal-form}
   For any $\phi\in \DRSLLang$ such that $\phi$ is not a standpoint sharpening, there exists some $\phi'\in\DRSLLang$ in normal form such that  for any preferentially closed set $\mathcal{A}\subseteq \DRSLLang$, $\phi\in\mathcal{A}$ iff $\phi'\in\mathcal{A}$. 
\end{lemma}

In our following results, we therefore assume without losing generality that any DRSL formula considered is in normal form. We utilise this in the following results leading up to our main representation theorem. In our representation result, we aim to take a preferentially closed set of DRSL statements $\A$ and show that there is a preferential standpoint structure $M$ which satisfies formulas iff they are in $\A$. In order to do this, we induce sets of formulas in $\KLMLang$ based on the formulas in $\A$ which are closed under propositional preferential reasoning. 

\begin{definition}\label{definition:As-derived-set}
    Given a set $\A\subseteq \DRSLLang$ and a standpoint symbol $s\in \mathcal{S}$, the derived set $\As\subseteq\KLMLang$ is given by 
    \[\As:=\{\phi\in\KLMLang\mid\Box_s\phi\in\A\}\]
\end{definition}

This intuitively defines the set defeasible beliefs which necessarily hold for the standpoint $s$. We additionally need to consider sets in $\KLMLang$ which consider additional information which is possible, but not necessary for a standpoint. We start with a preliminary definition.

\begin{definition}\label{def:Asipsi-defined-set}
    For a set $\A\subseteq \DRSLLang$ and a formula of the form $\Diamond_s\psi\in\A$ where $\psi\in \KLMLang$ we define a set of $(\psi,s)$-conjuncts in $\A$ as a finite set $C\subseteq\KLMLang$ such that $\psi\in C$ and $\Diamond_s(\bigwedge C)\in\A$. We denote the set of all $(\psi,s)$-conjuncts in $\A$ by $\textbf{Conj}_\A(\psi,s)$.
\end{definition}

This allows us to induce the sets derived from $\A$ which collect the local consequences relative to a formula of the form $\Diamond_s \phi$.

\begin{definition}\label{def:MaxConj-set}
    Suppose $\A\subseteq \DRSLLang$ and $\Diamond_s\psi\in\A$ where $\psi\in \KLMLang$. For each sequence of conjuncts $\omega=(C_i)_{i\in\mathbb{N}}$ in $\textbf{Conj}_\A(\psi,s)$ such that $C_i\subset C_{i+1}$ for all $i\in\mathbb{N}$, we define $\Asipsi$ as the maximal element of this chain. That is,
    \[\Asipsi:=\bigcup^\infty_{i=1}C_i\]

    for a fixed $\Diamond_s\psi\in\A$. We denote the set of all such union as $\textbf{LimConj}_\A(\psi,s)$. We further denote the subset-maximal elements of $\textbf{LimConj}_\A(\psi,s)$ as $\textbf{MaxConj}_\A(\psi,s)$.
\end{definition}

That is, if $X\in \textbf{MaxConj}_\A(\psi,s)$, then there is no $X'\in \textbf{LimConj}_\A(\psi,s)$ such that $X\subset X'$. Note that, in general we are guaranteed the existence a non-trivial infinite chain of conjuncts for each $\Diamond_s\psi\in\A$. This follows from the fact that there are infinitely many syntactically distinct Boolean formulas over a finite set of atoms, and so infinitely many formulas of the form $\alpha\twiddle \alpha$ will occur in $\Asipsi$, due to an application of the rules \textbf{Ref}, \textbf{RN} and \textbf{K.b}. We derive a set of formula in $\KLMLang$ for \textit{each chain} in order to maintain satisfiability. If instead we define the set $\Aspsi$ as the union of all conjuncts in $\textbf{Conj}_\A(\psi,s)$, we may obtain contradictory information in $\Aspsi$ that does not occur in $\A$. For example, we may have that $\Diamond_s(\psi\wedge p), \Diamond_s(\psi\wedge \neg p)\in \A$ and so the union of $\textbf{Conj}_\A(\psi,s)$ would contain $p$ and $\neg p$ and thus be unsatisfiable in non-trivial cases. However, $p$ and $\neg p$ would not occur in the same chain of conjuncts unless $\Diamond_s (\psi\wedge p\wedge \neg p)\in \A$, which would make $\A$ unsatisfiable. The following result shows that sets we derive from $\A$ preserve closure under the KLM postulates.

\begin{lemma}\label{lemma:derived-KLMprop-sets-of-closed-set-are-closed}
If $\A$ is preferentially closed, then each derived set of the form $\As$ is closed under the KLM postulates and classical propositional logic. Furthermore, for all $s\in \mathcal{S}$ and all $\Diamond \psi \in A$ where $\psi\in \KLMLang$, we have that each set in $\textbf{MaxConj.}_\A(\psi,s)$ is closed under the KLM postulates and classical propositional logic.
\end{lemma}

We also show that our derived sets preserve satisfiability.

\begin{lemma}\label{lemma:if-A-satisfiable-then-derived-KLM-sets-satisfiable}
    If $\A$ is satisfiable, then $\As$ and $\Asipsi\in \textbf{MaxConj.}_\A(\psi,s)$ are non-trivially satisfiable.
\end{lemma}

It follows that, when $\A$ is satisfiable and preferentially closed, for each derived set of the form $\As$ there exists some preferential interpretation $\I_s$ such that $\I_s\Vdash \phi$ iff $\phi\in\As$. Similarly, for each derived set of the form $\Asipsi$, there exists some preferential interpretation $\Isipsi$ such that $\Isipsi\Vdash \phi$ iff $\phi\in\Asipsi$. We use these derived preferential interpretations to define the preferential standpoint structure which characterises $\A$.

\begin{definition}\label{def:representative-preferential-model-for-A}
    Given a preferentially closed set $\mathcal{A}$ of DRSL formulas, we define the preferential  standpoint structure $M_{\mathcal{A}}=(\Pi,\sigma,\tau)$ as follows:
    \begin{enumerate}
        \item $\Pi=\{\pi_s\mid s\in\mathcal{S}\cup\{*\}\}\cup\{\pi_{s,\omega}^{\psi}\mid \Diamond_s\psi\in\A, \psi\in\KLMLang \text{ and } \Asipsi\in\textbf{MaxConj}_{\A}(s,\psi)\}$.
        \item $\sigma(s)=\{\pi_t\mid t\preceq s\in\mathcal{A}\}\cup \{\pi_{t,\omega}^\psi\mid t\preceq s\in \mathcal{A}\}$.
        \item $\tau(\pi_s)=\mathcal{I}_s$ and $\tau(\pi_{s,\omega}^\psi)=\Isipsi$.
    \end{enumerate}
\end{definition}

As a result of Lemma \ref{lemma:derived-KLMprop-sets-of-closed-set-are-closed}, $\tau(\pi)$ is a well-defined preferential interpretation for each $\pi\in\Pi$. Moreover, $\tau(s)$ is non-empty for all $s\in\mathcal{S}$ and $\tau(*)=\Pi$ since $s\preceq *\in \A$ for all $s\in\mathcal{S}$. That is, $M_\A$ is a well-defined preferential standpoint structure. The following lemma shows us that $M_\A$ is in fact the standpoint structure we use to represent the set $\A$.

\begin{lemma}\label{lemma:M_A-constructed-is-the-model-that-represents-A}
    For any satisfiable, preferentially closed set $\A$, we have that $\phi\in\A$ iff $M_\A\Vdash \phi$.
\end{lemma}

As a result of this, we obtain the following representation theorem, which is the main contribution of this section.

\begin{theorem}\label{theorem:main-representation-result}
    A set $\mathcal{A}\subseteq\DRSLLang$ of DRSL formulas is satisfiable and preferentially closed iff there exists some preferential standpoint structure $M$ such that $\phi\in\mathcal{A}$ iff $M\Vdash \phi$.
\end{theorem}

\section{Entailment}\label{section:entailments}

In this section, we turn our attention to defining and computing entailment from DRSL knowledge bases. In the propositional setting, KLM-style defeasible reasoning has several non-equivalent notions of entailment defined for a given knowledge base \cite{casini:beyondratclosure}. In this paper, we do not focus on a single form of defeasible entailment, but rather we attempt to show principled means  to lift classes of entailments from the propositional setting to the standpoints setting. The next sections characterise the lifting of several well-known defeasible entailment operators from the propositional case. We assume that each knowledge base is given in \textit{conjunction free normal form}. That is, for each formula in $\KB$ in normal form that has the shape of a conjunction $\bigwedge^n_{i=1}\phi_i$, we replace it by its set of conjuncts $\{\phi_i\mid 1\leq i\leq n\}$. It is clear to see that this is logically equivalent, and this assumption allows us simpler definitions for characterizing derived propositional knowledge bases within the following sections.

\subsection{Preferential Entailment}\label{subsection:preferential-entailment}

The first form of defeasible entailment we consider is preferential entailment, which is the monotonic core of KLM-style reasoning. This is defined via the usual Tarskian methods. We provide both the propositional and DRSL definition here.

\begin{definition}
    Given a (finite) knowledge base $\KB\subseteq\KLMLang$, we say that $\phi\in \KLMLang$ is \textit{preferentially entailed} by $\KB$, denoted $\KB\vDash_{P,prop} \phi$ if for any preferential interpretation $\I$, we have that $\I\Vdash \KB$ implies $\I\Vdash \phi$.
\end{definition}

\begin{definition}
    Given a (finite) knowledge base $\KB\subseteq\DRSLLang$, we say that $\phi\in \DRSLLang$ is \textit{preferentially entailed} by $\KB$, denoted $\KB\vDash_P \phi$ if for any preferential standpoint structure $M$, we have that $M\Vdash \KB$ implies $M\Vdash \phi$.
\end{definition}

From Theorem \ref{theorem:main-representation-result}, we obtain the following corollary.

\begin{corollary}\label{corollary:preferential-entailment-is-smallest-closed-set}
    $\phi\in \DRSLLang$ is preferentially entailed by $\KB$ iff $\phi\in C(\KB)$, where $C(\KB)$ is the set obtained by exhaustively applying the rules in Figures \ref{fig:DRSL-modal-rules} and \ref{fig:DRSL-KLM-postulates}, as well as the rules of classical propositional logic to $\KB$.
\end{corollary}

The above corollary shows that we are able to compute preferential entailment through applications of the rules given in Figures \ref{fig:DRSL-modal-rules} and \ref{fig:DRSL-KLM-postulates}, in order to determine whether $\KB\vDash_P \phi$ for some $\KB\subseteq\DRSLLang$ and $\phi\in \DRSLLang$. Next  we show that we can reduce entailment-checking in the standpoint case to entailment checking in the propositional case. We first note following preliminary lemma.

\begin{lemma}\label{lemma:entailment-for-standpoint-sharpenings}
    For any $\KB\subseteq\DRSLLang$, we have that a standpoint sharpening statement $s\preceq t\in C(\KB)$ iff  $t=*$ or $t=s$ or $s\preceq t$ is in the transitive closure of $\preceq_\KB=\{(s_1,s_2)\mid s_1\preceq s_2\in \KB\}$.
\end{lemma}

The above holds since it is clear from the rules in Figure \ref{fig:DRSL-modal-rules} that the only way to obtain standpoint sharpenings as conclusions are through one of the cases above. We are therefore able to compute the set of standpoint sharpening statements in \textsc{PTime}, since this amounts to computing transitive closure. With this in mind, we derive a set of KLM propositional statements from $\KB$ for each standpoint. Intuitively, this set represents the ``base case'' of beliefs for this standpoint.

\begin{definition}\label{definition:K_s-associated-knowledge-base}
    For a DRSL knowledge base $\KB\subseteq\DRSLLang$, and a standpoint $s\in\mathcal{S}$ we define the derived set $\KB_s$ as: 
    \[\KB_s:=\{\phi\in \KLMLang\mid \phi\in\KB \text{ or } \Box_t\phi\in\KB, s\preceq t\in C(\KB)\}.\]
\end{definition}

We then define another set of propositional KLM statements for each diamond-bound statement occurring in $\KB$. Intuitively, this represents constructing the precise ``version'' of a standpoint $s$'s beliefs once we take into account an extra possibility belief made explicit in the knowledge base.

\begin{definition}
    For a DRSL knowledge base $\KB\subseteq\DRSLLang$, a standpoint $s\in\mathcal{S}$ and a statement $\Diamond_s\phi\in \KB$, we define the set $\KB^\phi_s$ as:
    \[\KB^\phi_s:=\KB_s\cup\{\phi\}.\]
\end{definition}

Moreover, from this we can find a collection of sets in $\KLMLang$ which are all those that ought to be associated with a specific standpoint, representing each distinguished possibility for a standpoint which is found in the knowledge base.

\begin{definition}\label{definition:Prop-KB-s}
    For a knowledge base $\KB$ and a standpoint $s\in\mathcal{S}$, we define the set of \textit{s-associated propositional knowledge bases}, or $\textbf{PropKB}_{\KB}(s)$ as 

        \[\textbf{PropKB}_{\KB}(s):= \{\KB_t\mid t\preceq s\in C(\KB)\}
        \cup\{\KB^\phi_t\mid t\preceq s\in C(\KB), \Diamond
        _s\phi\in\KB\}.\]

\end{definition}

These derived sets are inspired by the sets constructed in the algorithm \texttt{StandpointSplit}, given by Leisegang et al.~\cite{LMR2024}. As a result, it is known that computing the sets $\KB_s$, $\KB^\phi_s$ and $\textbf{PropKB}_{\KB}(s)$ is in \textsc{PTime}. We then show how the sets defined above allow us to check for preferential entailment in the DRSL case.

\begin{proposition}\label{proposition:pref-entailment-of-box-statements-reducible}
Consider a DRSL knowledge base $\KB\subseteq\DRSLLang$, and a statement $\Box_s\psi\in \DRSLLang$, where $\psi\in\KLMLang$. Then, $\KB\vDash_P\Box_s\psi$ iff $\KB_s\vDash_{P,\textit{prop}} \psi$.
\end{proposition}

Note that this also allows us to check whether $\KB\vDash_P \psi$ for $\psi\in\KLMLang$, since we can equivalently check $\KB\vDash_P\Box_*\psi$. For diamond-bound statements, we cannot reduce DRSL preferential entailment-checking to a single propositional entailment check, but rather to a set of propositional entailment checks.

\begin{proposition}\label{proposition:preferential-entailment-diamond-statements}
    Consider a DRSL knowledge base $\KB\subseteq\DRSLLang$, and a statement $\Diamond_s\psi\in \DRSLLang$, where $\psi\in\KLMLang$. Then $\KB\vDash_P \Diamond_s\psi$ iff $X\vDash_{P,\text{prop}} \psi$ for some $X\in \textbf{PropKB}_{\KB}(s)$.
\end{proposition}

We can then, in general, reduce DRSL preferential entailment-checking to propositional preferential entailment-checking. Each DRSL formula can be rewritten in normal form. Then, in order to check such a conjunction in normal form, we just check each of the conjuncts using the corresponding propositional entailment checks from Propositions \ref{proposition:pref-entailment-of-box-statements-reducible} and \ref{proposition:preferential-entailment-diamond-statements}. Using this, we are able to analyse the complexity of preferential reasoning in DRSL, and show that it is within the same complexity class as preferential entailment for KLM propositional logic.

\begin{theorem}\label{theorem:complexity-of-DRSL-preferentail-entailment}
    Preferential entailment-checking in DRSL is \textsc{coNP}-complete.
\end{theorem}

\subsection{Entailment Based on Single Ranked Models}\label{subsection:ranked-entailments}

Next, we consider systematically a class of non-monotonic entailment relations considered in KLM-style defeasible reasoning. In particular, these are the defeasible entailment relations where the order on the preferential interpretation can be expressed by a \textit{ranking function}. This is an important class of relations, as many well-known non-monotonic entailment relations fall within this class, such as rational closure \cite{lehmann:conditionalentail, Pearl:SystemZ}, lexicographic closure \cite{lehmann:lexicographicreason} and any inference relation based on a single $c$-representation \cite{KernIsberner2001}.

\begin{definition}
 A \emph{ranking function} $r$ is a function $r:2^\mathcal{P}\rightarrow \mathds{N}\cup \{\infty\}$, satisfying the following property: if $r(u)<\infty$, then for every $0\leq j<r(u)$ there exists $v\in 2^\mathcal{P}$ such that $r(v)=j$.
\end{definition}

Each ranking function can be associated to a preferential interpretation by choosing $\I=(W,l,<_r)$ where $W=2^{\mathcal{P}}\setminus r^{-1}\{\infty\}$, $l$ is the identity and $u<_r v$ iff $r(u)<r(v)$. We write $r\Vdash \alpha\twiddle\beta$ iff. $min_{<_r}\llbracket\alpha\rrbracket\subseteq \llbracket\beta\rrbracket$. That is, if the minimally ranked $\alpha$-valuations satisfy $\beta$.
Given a knowledge base $\KB\in\KLMLang$ and a formula $\phi\in\KLMLang$, we consider the non-monotonic entailment relations $\dentails$ such that $\KB\dentails \phi$ iff $r\Vdash \phi$, where $r$ is a single ranking function which satisfies $\KB$. This builds on the work of Casini et al.~\cite{casini:beyondratclosure}, who analyse algorithmic and semantic approaches to this general class of entailments. In particular, they introduce the \texttt{DefeasibleEntailment} algorithm which takes a ranking function, and provides an algorithmic means for computing the defeasible entailment relation which is based on this ranking. Our work can be seen as a significant extension of the work of Leisegang et al.~\cite{LMR2024} who consider Rational Closure in the context of DRSL. In the following section, we provide a systematic means for lifting both the semantics and algorithms for  entailments within this class. Furthermore, we will show that the proposed algorithms and semantics are equivalent definitions for determining such an entailment. Finally, we show that entailment-checking for this class of relations in the DRSL case stays within the same complexity class as entailment checking in the propositional case. In our analysis of complexity, we note that \texttt{DefeasibleEntailment} may require exponentially many SAT calls in the size of $\KB$ \cite{casini:beyondratclosure}. This is under the assumption that $r_\KB$ is already computed, and in general there are no known complexity bounds for the construction $r_\KB$ from $\KB$. In the well-studied cases of rational closure and lexicographic closure, entailment is $\textsc{P}_{\parallel}^\textsc{NP}$-complete and $\textsc{P}^\textsc{NP}$-complete, respectively \cite{lehmann:conditionalentail,LucasiewiczEiter:ComplexityFromDefaultReasoning}. In this section, we make the following two assumptions about a given propositional defeasible entailment relation $\dentails$:

\begin{enumerate}

    \item We assume that, given a KB $\KB\in \KLMLang$, there exists a deterministic means to construct our ranking function $r_\KB$ such that $\KB\dentails \phi$ iff $r_\KB \Vdash \phi$. That is, there exists some known mapping $r:2^{\KLMLang}\rightarrow\mathscr{R}$ which defines $\dentails$, where $\mathscr{R}$ is the set of ranking functions. We refer to such a mapping as a \textit{ranking strategy}, and denote $r(\KB)$ as $r_{\KB}$.
    \item We assume that $\phi\in\KB$ implies $\KB\dentails \phi$ and if $\phi$ is a Boolean formula $\KB\vDash_{P,prop}\phi$ iff $\KB\dentails \phi$. These properties are known as \textit{Inclusion} and \textit{Classical Preservation}, and are considered necessary for a well-defined defeasible entailment \cite{casini:beyondratclosure}.
    \end{enumerate}

For such an entailment relation $\dentails_{prop}$, our goal is to define an entailment relation $\dentails_{DRSL}$ which faithfully lifts the entailment relation from $\KLMLang$ to $\DRSLLang$. In order to understand this, we once again consider the approach we take in  the case of preferential reasoning. In Definition \ref{definition:Prop-KB-s}, $\textbf{Prop}_{\KB}(s)$ refers to a number of different propositional knowledge bases which represent salient sets of beliefs which ought to be considered for each standpoint. Here, $\KB_s$ provides the basic necessary set of knowledge that this standpoint is required to have, while each $\KB_s^\phi$ is required in order to combine the basic beliefs of $s$ with additional possibilities which are \textit{made distinct} through their inclusion in the knowledge base. Furthermore, if we let $r$ be the ranking strategy associated with $\dentails_{prop}$, we are able to semantically realise the sets of beliefs in one of the knowledge bases $X\in \textbf{Prop}_{\KB}(s)$ through the associated ranking function $r_{X}$. Using this motivation, we define a ranked standpoint structure in which each precisification is exactly the ranking function associated to some knowledge base in $\textbf{Prop}_{\KB}(*)$, and the set of precisifications assigned to each standpoint $s$ is exactly those precisifications which model the different sets of beliefs in $\textbf{Prop}_{\KB}(s)$. 

\begin{definition}\label{def:model-for-defeasibleentailment-dentails}
    Consider a DRSL knowledge base $\KB\subseteq \DRSLLang$, a propositional defeasible entailment $\dentails$ and its associated ranking strategy $r$. We define the \textit{standpoint $\dentails$-model for $\KB$} as $M_{\KB}^{\dentails}=(\Pi_{\KB},\sigma_{\KB},\tau^r_{\KB})$ where
    \begin{itemize}
        \item $\Pi_{\KB}:=\{\pi_X\mid X\in\textbf{Prop}_{\KB}(*)\setminus\{\KB_*\} \}$.
        \item $\sigma_\KB(s)=\{\pi_X\in \Pi_{\KB}\mid X\in \textbf{Prop}_{\KB}(s)\setminus\{\KB_*\}\}$ for $s\in \mathcal{S}$.
        \item $\tau^r_{\KB}(\pi_X)=r_X$, where $r_X$ is the ranking function determined by the ranking strategy $r$.
    \end{itemize}
\end{definition}

Note here that the definitions of $\Pi_{\KB}$ and $\sigma_{\KB}$ depends only on $\KB$ and not on the particular entailment or ranking strategy which we consider, and therefore we omit $r$ from their indices.
The first basic results to note is that $M_{\KB}^{\dentails}$ is a model of $\KB$.

\begin{proposition}\label{proposition:ranked-entailment-preserves-satisfiability-and-contains-the-KB}
Consider $\KB\subseteq \DRSLLang$, a defeasible entailment $\dentails$ and the associated ranking strategy $r$:
\begin{enumerate}
    \item If $\KB$ is satisfiable, then there is no $\pi_X\in \Pi_{\KB}$ such that $\tau_{\KB}^r(\pi_X)=\emptyset$.
    \item If $\phi\in\KB$ then $M_{\KB}^{\dentails}\Vdash \phi$.
\end{enumerate}
\end{proposition}

This model defines a means for lifting a propositional defeasible entailment $\dentails_{prop}$ to the DRSL case. 

\begin{definition}
    For $\KB\subseteq \DRSLLang$ and $\phi\in \DRSLLang$, we say that $\KB \dentails_{DRSL} \phi$ iff $M^{\dentails_{prop}}_{KB}\Vdash \phi$.
\end{definition}

For the rest of the section we abuse notation by using $\dentails$ to refer to both the propositional entailment and its lifting to DRSL. An important design choice to discuss here is the fact that in our model, we remove $\KB_*$ from $\textbf{Prop}_{\KB}(*)$. We refer to this as the \textit{closed world assumption for standpoints}:

   The \emph{closed world assumption for standpoints} is the assumption that there exist no novel standpoints outside those that occur in our vocabulary. This is what allows us to faithfully construct $M_{\KB}^{\dentails}$ as a representative model for entailment, since we can create only precisifications that occur in standpoints that are named in our knowledge base. In particular, in both the algorithmic and semantic lifting of $\dentails$ to the propositional case, we remove $\KB_*$ from the considered propositional knowledge bases. This is because $\KB_*$ only contains those statements which we know are universally true, and therefore, emulates a standpoint where nothing except universal beliefs are known. We make this design choice since when we are working with defeasible entailment, we are reasoning non-monotonically about the beliefs of known standpoints, and so we believe this assumption is a valid one. In fact, the alternative may lead us to disregard conclusions which seem reasonable in a setting of prototypical reasoning. Consider the knowledge base $\KB=\{\Box_s(p\twiddle q),\Box_t(p\twiddle q)\}$. It seems reasonable here to say that all considered standpoints believe that $p\twiddle q$ holds. That is $\KB\dentails \Box_*(p\twiddle q)$. However, if we were to require that $\KB_*$ is considered, then we must consider whether the empty set entails $p\twiddle q$, which it usually does not for well-known definitions of $\dentails$. Hence, $\KB\ndentails \Box_*(p\twiddle q)$, simply because $\KB_*$ is empty. This seems undesirable, since our logic ultimately looks to increase possible agreements between standpoints, if possible, and withdraw upon learning new information which contradicts this (for example, by adding $v$ to $\mathcal{S}$ and $\Box_v(p\twiddle\neg q)$ to $\KB$). The same assumption is not made when $\Diamond_*\psi$ is introduced to the knowledge base. In this case, we still consider $\KB^\psi_*$, since our knowledge informs us that some standpoint holds $\psi$ possible, but we cannot be more specific about which standpoint that is. A result of this is that it is possible to emulate the open world setting using our framework by simply adding $\Diamond_*\top$ to a given DRSL knowledge base. This necessitates the consideration of the knowledge base $\KB^\top_*$ in the semantics, which is deductively equivalent to $\KB_*$. That is, it introduces a ``place-holder'' precisification, whose only beliefs are those we know to be universal across our whole standpoint domain, and therefore includes the possibility of unnamed standpoints where no details of their beliefs are known.

One corollary of this, is that the entailment for DRSL is not well-defined when our knowledge base is propositional. If $\KB\subseteq \KLMLang$, then for any defeasible entailment $\Pi_\KB=\emptyset$ and so $M^{\dentails}_\KB$ is no longer a well-defined standpoint structure. However, we note that in cases where the knowledge base is entirely propositional, we can refer to well-known propositional approaches to KLM, or we can avoid this by simulating the open world scenario by adding $\Diamond_*\top$ to our knowledge base. With this addition, we get the expected result that $\dentails_{DRSL}$ collapses into $\dentails_{prop}$ when considering propositional statements.

\begin{proposition}\label{proposition:defeasible-entailment-collapses-into-propositional-case-open-world}
    Suppose $\KB=\KB'\cup\{\Diamond_*\top\}$ where $\KB'\subseteq\KLMLang$. Then for $\phi\in \KLMLang$, we have $\KB\dentails_{DRSL} \phi$  iff $\KB'\dentails_{prop} \phi$.
\end{proposition}

\begin{example}

We show here the model construction of the rational closure of $\KB_T$ in Example \ref{example:tomato-model}. This is given by the model $M_{RC}=(\Pi,\sigma,\tau)$ where $\Pi=\{\pi_B,\pi_C,\pi_L\}$; $\sigma(B)=\{\pi_B\}$, $\sigma(C)=\{\pi_C,\pi_L\}$, $\sigma(L)=\{\pi_L\}$; $\tau(\pi_B)$ is given by a ranking function where $r(v)=0$ if $v\Vdash(t\rightarrow f)\wedge (f\rightarrow v)$ and $r(v)=\infty$ otherwise; $\tau(\pi_C)$ and $\tau(\pi_L)$ are defined as ranking functions in the table below:

\hspace{5pt}
\begin{center}
{   \small
      \begin{tabular}
    {|c|c|c|}
      \hline  
      \textbf{rank} & \textbf{$\tau(\pi_C)$ }& \textbf{$\tau(\pi_L)$}\\
      \hline
      $\infty$ & all other valuations & all other valuations  \\
      \hline
        \multirow{2}{*}{1} & $\{ts_a vs_wf\}$, $\{ts_wf\}$,
         & \multirow{2}{*}{$\{ts_wf\}$, $\{t\}$}\\
         & $\{t\}$, $\{s_a vs_wf\}$ & \\
        \hline
        0 & $\{ts_a v\}$, $\{s_a v\}$,$\{s_wf\}$, $\emptyset$ & $\{ts_a v\}$, $\{s_a v\}$,$\{s_wf\}$, $\emptyset$\\
        \hline
    \end{tabular}}
  \end{center}  
\hspace{5pt}

We can see that $M\Vdash \Box_C(t\twiddle \neg f)$ and so $\KB\dentails_{RC} \Box_C(t\twiddle \neg f)$, where $\dentails_{RC}$ denotes rational closure entailment. Moreover, we can see by the counter-model in Example \ref{example:tomato-model} that $\KB_T\nvDash_P \Box_C(t\twiddle \neg f)$, showing that $\dentails_{RC}$ is stictly stronger than $\vDash_P$.
\end{example}

Now that we have defined a semantic means for extending ranking-based defeasible entailment from the propositional to the DRSL case, we analyse algorithmic approaches to DRSL, and show that we can lift entailment algorithms from the propositional case to the DRSL case. The construction of such algorithms is linked closely to the semantic structures previously defined. Since our basis for lifting entailment to a DRSL knowledge base involves utilising the propositional knowledge bases in $\textbf{Prop}_{\KB}(*)$, we can similarly use an algorithm which queries defeasible entailment in DRSL by using the original propositional algorithms applied to appropriate knowledge bases in $\textbf{Prop}_{\KB}(*)$. This is defined by the algorithm $\texttt{StdptRankEntail}$ in Figure \ref{fig:ranked-entailment-DRSL-main}, which is a generalization of the DRSL algorithm for rational closure proposed by Leisegang et al. \cite{LMR2024}. The algorithm works by directly checking a query with respect to the underlying propositional algorithm given for $\dentails$, where we check a query of the form $\Box_s\phi$ or $\Diamond_s\phi$ by checking the knowledge bases in $\textbf{Prop}_{\KB}(s)$. Note here that unlike the case of $\vDash_P$, we cannot check $\Box_s\phi$ queries by simply checking with respect to the knowledge base $\KB_s$. This is due to the non-monotonicity of $\dentails$, which means that there might be some $\KB_s^\psi$ and $\alpha\in\KLMLang$ such that $\KB_s\dentails \alpha$ and $\KB_s^\psi\ndentails \alpha$. We also note that in its most general form, our algorithm takes the ranking strategy $r$ as an input. For known entailment relations such as rational and lexicographic closure, we replace the call to $\texttt{DefeasibleEntail}$ with calls to the propositional rational and lexicographic closure algorithms respectively. Lastly, we note the abuse of notation here that in the original case, $\texttt{DefeasibleEntail}$ is only defined for single defeasible implications ``$\alpha\twiddle\beta$'', while in our case we allow conjunctions of such implications as inputs. In this case, we treat computing $\texttt{DefeasibleEntail}$ for a conjunction of implications as equivalent to computing $\texttt{DefeasibleEntail}$ for each of its conjuncts. The algorithm structures itself similarly to the semantic model in Definition \ref{def:model-for-defeasibleentailment-dentails}: we split the model into a series of precisifications whose valuations are determined by propositional knowledge bases. On the other hand we split the algorithm into a series of calls to algorithms defined in the propositional case based on the same set of knowledge bases. This leads us to the  following correspondence result.

\begin{figure}
  \begin{algorithm}[H]
  \small
		\caption{StdptRankEntail}
			\textbf{Input}: A DRSL knowledge base $\mathcal{K}$ in normal form, a ranking strategy $r$ and a (non-standpoint sharpening) DRSL $\phi$ in normal form.\\
			\mbox{\textbf{Output}: True if $\KB\dentails\phi$; False otherwise.}\\[-1ex]
			\begin{algorithmic}[1]	\IF{$\phi=\phi_1\wedge\phi_2$}
				\IF{StdptRankEntail($\KB,r,\phi_1$)$\,=\,$True and\\ \ \ \ \  StdptRankEntail($\KB,r,\phi_2$)$\,=\,$True}
				\STATE\textbf{return} True;
				\ELSE 
				\STATE\textbf{return} False;
				\ENDIF
				\ELSIF{$\phi=\Box_s\psi$}
				\FOR{$X\in\textbf{PropKB}_{\KB}(s)\setminus \{K_*\}$}
				\IF{DefeasibleEntail($X,r_X,\psi$)=False}
				\STATE \textbf{return} False;
				\ENDIF
				\ENDFOR
				\STATE \textbf{return} True;
				\ELSIF{$\phi=\Diamond_s\psi$}
				\FOR{$X\in\textbf{PropKB}_{\KB}(s)\setminus \{K_*\}$}
				\IF{DefeasibleEntail($X,r_X,\psi$)=True}
				\STATE \textbf{return} True;
				\ENDIF
				\ENDFOR
                \ELSE
                \STATE\textbf{return} False; 
				\ENDIF
		\end{algorithmic}
		\end{algorithm}
    \caption{Ranked Entailment Algorithm for DRSL}
    \label{fig:ranked-entailment-DRSL-main}
\end{figure}

\begin{theorem}\label{theorem:algorithmic-representation-result}
    Given DRSL knowledge base $\KB$ with $\KB\nsubseteq \KLMLang$, a non-sharpening formula $\phi\in\DRSLLang$ and a defeasible entailment $\dentails$ with selection strategy $r$, we have $\KB\dentails \phi$ iff $\texttt{StdptRankEntail}(\KB,r,\phi)=\text{True}$.
\end{theorem}

For standpoint sharpening statements, we once again only entail those which are noted in Lemma \ref{lemma:entailment-for-standpoint-sharpenings}. This is clear from the construction of $M_\KB^{\dentails}$. Lastly, we consider the complexity of the algorithm. The algorithm above performs polynomially many calls to an underlying propositional algorithm, and the knowledge bases input in these calls are no bigger than the original knowledge base. Furthermore, these calls do not depend on outcomes of previous calls and can be parallelised. Therefore we obtain the following complexity result.

\begin{theorem}\label{theorem:complexities-of-defeasible-entailments}
    For any propositional defeasible entailment $\dentails_{prop}$, we have that if $\dentails_{prop}$ is computable in a complexity class $C$ such that $\textsc{P}_{\parallel}^\textsc{NP}\subseteq C$, then entailment-checking for $\dentails_{DRSL}$ remains in $C$. In particular, entailment checking for rational closure in DRSL is $\textsc{P}_{\parallel}^\textsc{NP}$-complete and entailment checking for lexicographic closure in DRSL is $\textsc{P}^\textsc{NP}$-complete.
\end{theorem}

Therefore, we are able to lift ranking based entailments from the propositional case to DRSL without increasing complexity in the cases of rational and lexicographic closure. More generally, if the underlying propositional entailment is in a class no better than $\textsc{P}_{\parallel}^\textsc{NP}$, the complexity is preserved in the DRSL case.

\section{Related Work and Conclusions}\label{section:relatedwork}

DRSL and its semantics were originally considered by Leisegang et al. \cite{LMR2024}. Similar notions of non-monotonic standpoint logics with default-style beliefs are considered by Gorczyca and Straß \cite{GS2026,gorczyca-strass:nmonotonic-standpoint-s4f} who consider standpoint modalities in the non-monotonic modal logic S4F. Another method of integrating standpoint logics with KLM defeasibility is given by Leisegang et al. \cite{JELIA-NickIvanTommie}, who introduce Propositional Defeasible Standpoint Logic (PDSL) in which a defeasible notion of standpoint modalities and standpoint sharpenings are introduced. However, while defeasible implications occur within PDSL, they act as outer-level implications between modal statements, rather than defeasible beliefs held by standpoints. Hence, the semantics for PDSL and DRSL are non-equivalent. KLM-style defeasible reasoning has been employed in other modal logics such as \textbf{K} \cite{britzvarzin:defeasiblemodalities} and linear temporal logic (LTL) \cite{chafik:defeasiblelineartemporal}. The addition of standpoint modalities to monotonic modal logics, such as LTL, has also been considered \cite{Aghamovetal25,DemriW24,alvarezlyon:stndpttemporal}.

In this paper, we provided a study of integrating standpoint modalities and KLM-style defeasible beliefs, extending the work by Leisegang et al. \cite{LMR2024} on DRSL. In particular, this paper contributes a KLM-style representation result for a set of proof-theoretic modal and KLM-style postulates which are sound and complete with respect to the semantics of preferential standpoint structures. Moreover, any set of DRSL statements closed under our postulates can be represented by a unique preferential standpoint structure. We then characterised preferential entailment in DRSL, as well as providing a systematic means for lifting the class of single ranking function based defeasible entailment relations from propositional to standpoint logics. This includes well-known relations such as rational and lexicographic closure. Lastly, for all defeasible entailments considered, we showed that entailment-checking in the DRSL case falls within the same complexity class as the propositional case. 


\begin{acknowledgments}
    We would like to thank the School of Embedded Composite Artificial Intelligence (SECAI) -- project 57616814 funded by BMBF (the Bundesministerium für Bildung und Forschung) and DAAD (German Academic Exchange Service) - who funded a research visit to TU Dresden for Nicholas Leisegang which made this collaboration possible. This work is based on the research supported in part by the National Research Foundation of South Africa (REFERENCE NO: SAI240823262612).
\end{acknowledgments}

\section*{Declaration on Generative AI}
  The authors have not employed any Generative AI tools.

\bibliography{defeasiblestandpointlogics-2}

@article{kraus:nonmonotonic,
  title={Nonmonotonic reasoning, preferential models and cumulative logics},
  author={Kraus, Sarit and Lehmann, Daniel and Magidor, Menachem},
  journal={Artificial intelligence},
  volume={44},
  number={1-2},
  pages={167--207},
  year={1990},
  publisher={Elsevier}
}

@article{lehmann:conditionalentail,
  title={What does a conditional knowledge base entail?},
  author={Lehmann, Daniel and Magidor, Menachem},
  journal={Artificial intelligence},
  volume={55},
  number={1},
  pages={1--60},
  year={1992},
  publisher={Elsevier}
}

@inproceedings{alvarezrudolph:propositionalstdpt,
  author    = {Lucía {G{\'o}mez Álvarez} and  Sebastian Rudolph},
  title     = {Standpoint Logic: Multi-Perspective Knowledge Representation},
  booktitle = {Formal Ontology in Information Systems - Proceedings of the Twelfth
                  International Conference, {FOIS} 2021, Bozen-Bolzano, Italy, September
                  11-18, 2021},
  series    = {Frontiers in Artificial Intelligence and Applications},
  volume    = {3344},
  publisher = {IOS Press},
  year      = {2021},
  pages     = {3 - 17}
}

@inproceedings{alvarez:stdptlogicfocase,
  author    = {Luc{\'{\i}}a {G{\'{o}}mez {\'{A}}lvarez} and Sebastian Rudolph and
               Hannes Strass},
  title     = {How to Agree to Disagree: Managing Ontological Perspectives using
               Standpoint Logic},
  booktitle = {Proceedings of the 21st International Semantic Web Conference
               (ISWC 22)},
  series    = {Lecture Notes in Computer Science},
  volume    = {13489},
  publisher = {Springer},
  year      = {2022},
  month     = {October}
}

@InProceedings{casini:beyondratclosure,
author="Casini, Giovanni
and Meyer, Thomas
and Varzinczak, Ivan",
editor="Calimeri, Francesco
and Leone, Nicola
and Manna, Marco",
title="Taking Defeasible Entailment Beyond Rational Closure",
booktitle="Logics in Artificial Intelligence",
year="2019",
publisher="Springer International Publishing",
address="Cham",
pages="182--197",
isbn="978-3-030-19570-0"
}

@article{lehmann:lexicographicreason,
  title={Another perspective on default reasoning},
  author={Lehmann, Daniel},
  journal={Annals of mathematics and artificial intelligence},
  volume={15},
  pages={61--82},
  year={1995},
  publisher={Springer}
}

@inproceedings{alvarezrudolphstrass:standpointEL,
  title     = {Tractable Diversity: Scalable Multiperspective Ontology Management via Standpoint {EL}},
  author    = {Gómez Álvarez, Lucía and Rudolph, Sebastian and Strass, Hannes},
  booktitle = {Proceedings of the Thirty-Second International Joint Conference on
               Artificial Intelligence, {IJCAI-23}},
  publisher = {International Joint Conferences on Artificial Intelligence Organization},
  editor    = {Edith Elkind},
  pages     = {3258--3267},
  year      = {2023},
}

@article{chafik:defeasiblelineartemporal,
	author = "Anasse Chaﬁk and Fahima Cheikh-Alili and Jean-François Condottaa and Ivan Varzinczak",
	title = "Defeasible linear temporal logic",
	journal = "Journal of Applied Non-Classical Logics",
	volume = "33",
	number = "1",
	pages = "1--51",
	year = "2023"
}

@article{britzvarzin:defeasiblemodalities,
	author = "Katarina Britz and Ivan Varzinczak",
	title = "From {KLM}-style conditionals to defeasible modalities, and
	back",
	journal = "Journal of Applied Non-Classical Logics",
	volume = "28",
	number = "1",
	pages = "92--121",
	year = "2018"
}

@inproceedings{alvarezlyon:stndpttemporal,
  author       = {Nicola Gigante and
                  Lucía {Gómez Álvarez} and
                  Tim S. Lyon},
  title        = {Standpoint Linear Temporal Logic},
  booktitle    = {Proceedings of the 20th International Conference on Principles of
                  Knowledge Representation and Reasoning, {KR} 2023, Rhodes, Greece,
                  September 2-8, 2023},
  pages        = {311--321},
  year         = {2023}
}

@inproceedings{Pearl:SystemZ, author = {Pearl, Judea}, title = {System {Z}: a natural ordering of defaults with tractable applications to nonmonotonic reasoning}, year = {1990}, publisher = {Morgan Kaufmann Publishers Inc.}, address = {San Francisco, CA, USA}, booktitle = {Proceedings of the 3rd Conference on Theoretical Aspects of Reasoning about Knowledge}, pages = {121–135}, numpages = {15}, location = {Pacific Grove, California}, series = {TARK '90} }

@inproceedings{LMR2024,
  author    = {Nicholas Leisegang and Thomas Meyer and Sebastian Rudolph},
  title     = {Towards Propositional {KLM-Style} Defeasible Standpoint Logics},
  editor    = {Aurona Gerber and Jacques Maritz and Anban W. Pillay},
  booktitle = {Proceedings of the 5th Southern African Conference on {AI}
               Research (SACAIR'24)},
  series    = {CCIS},
  volume    = {2326},
  publisher = {Springer},
  year      = {2024},
  pages     = {459{\textendash}475},
}

@article{LucasiewiczEiter:ComplexityFromDefaultReasoning, 
author = {Eiter, Thomas and Lukasiewicz, Thomas}, 
title = {Default reasoning from conditional knowledge bases: complexity and tractable cases}, 
year = {2000}, 
publisher = {Elsevier Science Publishers Ltd.}, 
address = {GBR}, 
volume = {124}, 
number = {2}, journal = {Artif. Intell.},
pages = {169–241},
}

@inproceedings{GS2026,
  author    = {Piotr Gorczyca and Hannes Stra{\ss}},
  title     = {Non-Monotonic {S4F} Standpoint Logic},
  booktitle = {Proceedings of the 40th Annual {AAAI} Conference on Artificial
               Intelligence (AAAI-26)},
  year      = {2026},
  month     = {January}
}

@book{KernIsberner2001,
	author = {Gabriele Kern{-}Isberner},
	editor = {},
	publisher = {Springer Verlag},
	title = {Conditionals in Nonmonotonic Reasoning and Belief Revision: Considering Conditionals as Agents},
	year = {2001}
}

@inproceedings{JELIA-NickIvanTommie,
  author       = {Nicholas Leisegang and
                  Thomas Meyer and
                  Ivan Varzinczak},
  editor       = {Giovanni Casini and
                  Besik Dundua and
                  Temur Kutsia},
  title        = {Extending Defeasibility for Propositional Standpoint Logics},
  booktitle    = {Logics in Artificial Intelligence - 19th European Conference, {JELIA}
                  2025, Kutaisi, Georgia, September 1-4, 2025, Proceedings, Part {II}},
  series       = {Lecture Notes in Computer Science},
  volume       = {16094},
  pages        = {43--57},
  publisher    = {Springer},
  year         = {2025},
}

@inproceedings{gorczyca-strass:nmonotonic-standpoint-s4f,
  author    = {Piotr Gorczyca and Hannes Stra{\ss}},
  title     = {Adding Standpoint Modalities to Non-Monotonic {S4F:} Preliminary
               Results},
  editor    = {Nina Gierasimczuk and Jesse Heyninck},
  booktitle = {Proceedings of the 22nd International Workshop on Non-Monotonic
               Reasoning},
  year      = {2024},
  month     = {November}
}

@inproceedings{DemriW24,
  author       = {St{\'{e}}phane Demri and
                  Przemyslaw Andrzej Walega},
  editor       = {Ulle Endriss and
                  Francisco S. Melo and
                  Kerstin Bach and
                  Alberto Jos{\'{e}} Bugar{\'{\i}}n Diz and
                  Jose Maria Alonso{-}Moral and
                  Sen{\'{e}}n Barro and
                  Fredrik Heintz},
  title        = {Computational Complexity of Standpoint {LTL}},
  booktitle    = {{ECAI} 2024 - 27th European Conference on Artificial Intelligence,
                  19-24 October 2024, Santiago de Compostela, Spain - Including 13th
                  Conference on Prestigious Applications of Intelligent Systems {(PAIS}
                  2024)},
  series       = {Frontiers in Artificial Intelligence and Applications},
  volume       = {392},
  pages        = {1206--1213},
  publisher    = {{IOS} Press},
  year         = {2024},
}

@inproceedings{Aghamovetal25,
    title     = {{Model Checking Linear Temporal Logic with Standpoint Modalities}},
    author    = {Aghamov, Rajab and Baier, Christel and Karimov, Toghrul and Majumdar, Rupak and Ouaknine, Joël and Piribauer, Jakob and Spork, Timm},
    booktitle = {{Proceedings of the 22nd International Conference on Principles of Knowledge Representation and Reasoning}},
    pages     = {2--11},
    year      = {2025},
    month     = {10}
  }

\newpage
\appendix
\section{Proofs of Results in Section \ref{section:preferential-DRSL}}

\textbf{Lemma \ref{theorem:main-soundness}.}
\textit{    Any preferential standpoint structure satisfies the rules given in Figures \ref{fig:DRSL-modal-rules} and \ref{fig:DRSL-KLM-postulates}, and satisfies classical propositional logic.}

\begin{proof}
    We note here that the soundness of each rule follows similar patterns and standard techniques (specifically concerning well-known modal axioms such as \textbf{RN}) and so we restrict ourselves to a proof of \textbf{P.a} and \textbf{LLE.a.} as examples of the proof techniques that can be used for other rules in Figures \ref{fig:DRSL-modal-rules} and \ref{fig:DRSL-KLM-postulates}. We also note that closure under classical reasoning follows from the fact that each preferential interpretation satisfies classical reasoning. Therefore, for each $\pi\in \Pi$ we have that $\tau(\pi)$ satisfies any Boolean tautology and modus ponens (sufficient for classical logic), and so the preferential standpoint structure $M$ as a whole also satisfies each Boolean tautologies and modus ponens.

    \begin{itemize}
        \item For \textbf{P.a} assume we have a preferential standpoint structure $M=(\Pi,\sigma,\tau)$ such that  $M\Vdash s\preceq t$ and $M\Vdash \Box_t \phi$. Then for any $\pi\in \sigma(s)$ we have by $M\Vdash s\preceq t$ that $\sigma(s)\subseteq \sigma(t)$ and so $\pi\in \sigma(t)$. Then by $M\Vdash \Box_t \phi$ we have $M,\pi'\Vdash \phi$ for all $\pi'\in \sigma(t)$ and specifically $M,\pi\Vdash \phi$. But then $M,\pi\Vdash \phi$ for all $\pi\in\sigma(s)$ and so $M\Vdash \Box_s\phi$.
        \item For \textbf{LLE.a.} assume we have a preferential standpoint structure $M=(\Pi,\sigma,\tau)$ such that  $M\Vdash \Box_s(\alpha\leftrightarrow \beta)$ and $M\Vdash \Box_s(\alpha\twiddle \gamma)$. Then, for each $\pi\in\sigma(s)$ we have that $\tau(\pi)\Vdash \alpha\leftrightarrow \beta$ and $\tau(\pi)\Vdash \alpha\twiddle \gamma$. Then, since $\tau(\pi)$ is a preferential interpretation, by propositional \textbf{LLE} we have $\tau(\pi)\Vdash \beta\twiddle \gamma$. Then since $\pi$ is a randomly chosen member of $\sigma(s)$ we conclude that $M\Vdash \Box_s(\beta\twiddle\gamma)$.
    \end{itemize}

\end{proof}

\textbf{Lemma \ref{lemma:proof-theoretic-normal-form}.}\textit{ Any DRSL formula which is not a standpoint sharpening can be expressed equivalently as a formula in normal form. That is, for any $\phi\in \DRSLLang$ and any preferentially closed set $\mathcal{A}\subseteq \DRSLLang$, there exists some $\phi'\in\DRSLLang$ in normal form such that $\phi\in\mathcal{A}$ iff $\phi'\in\mathcal{A}$.}

\begin{proof}
    Let $\phi$ be any DRSL formula. Then wlog we can assume it is in the form $\phi=\bigwedge^m_{j=1}\psi_j$ where for each $j$, either $\psi_j\in \KLMLang$ or $\psi_j$ is bound by a standpoint modality. We then describe algorithmically how we can convert each formula into a formula $\phi'$ in normal form using our postulates. We also justify that each step in our process can be reversed in the calculus and hence show that the formula reduced to normal form is not only a consequence of the original formula but is equivalent to it. That is $\phi\in\mathcal{A}$ iff $\phi'\in\mathcal{A}$.

    Initially, we break $\phi$ into its set of conjuncts and apply rules to them one by one (we can do this as a well-known property of classical logic).

    For each conjunct $\psi_j$,
 \begin{itemize}
     \item[A.] If $\psi_j\in\KLMLang$, then we do not alter it.
     \item[B.] Otherwise $\psi_j$ is bound by some standpoint modality and is in the form $\#_s\xi$ where $\#_s\in\{\Box_s,\Diamond_s\}$ and $\xi\in\DRSLLang$. We consider three sub-cases:
     \begin{enumerate}
         \item  $\xi\in\KLMLang$. Then we do not alter it.
         \item $\xi$ is bounded by some standpoint modality (that is $\xi=\#'_t\chi$). Then if $\#_s=\Box_s$ we remove the outer modality by applying \textbf{D} and then \textbf{4.b.} if $\#'_t=\Diamond_t$ or \textbf{5.b.} if $\#'_t=\Box_t$. If $\#_s=\Diamond_s$ we can simply apply \textbf{4'.} if $\#'_t=\Diamond_t$ or \textbf{5'.} if $\#'_t=\Box_t$. We can repeat this process finitely many times until we only have one modality on the outermost level. Then if we obtain something of the form $\#_{t}\zeta$ with $\zeta\in\KLMLang$ we do not alter it. If $\zeta\notin\KLMLang$ then it must be a conjunction of more than one DRSL formulas and we apply step 3. Note that each step here can be reversed. In order to reintroduce a modality of the form $\Box_s$ to the formula we can apply \textbf{RN}, and in order to introduce a modality $\Diamond_s$ to the outside of the formula we can apply \textbf{RN} and then \textbf{D}.
         \item $\xi$ is a conjunct of 2 or more DRSL formulas. That is $\xi=\bigwedge^l_{k=1}\chi_k$ where $l\geq 2$. Then we consider two more subcases:
         \begin{itemize}
             \item[a.] If $\#_s=\Box_s$ then we apply \textbf{$\Box$-Dist.a.} finitely many times until we distribute $\Box_s$ across each conjunct and obtain $\psi'_j=\bigwedge^l_{k=1}(\Box_s\chi_k)$. Then we apply our algorithm to each conjuct $\Box_s\chi_k$.  This step can be reversed by applying  \textbf{$\Box$-Dist.b.} finitely many times.
             \item[b.]  If $\#_s=\Diamond_s$ then we do the following. For each $\chi_k\notin\KLMLang$ we apply \textbf{$\Diamond$-Dist.} finitely many times until we obtain \[\psi'_j=\bigwedge_{k\in\{1,..,l\},{\chi_k}\notin\KLMLang}(\Diamond_s\chi_k)\wedge \Diamond_s(\bigwedge_{k\in\{1,..,l\},{\chi_k}\in\KLMLang}\chi_k)\] 
         Then we apply our algorithm to each conjunct $\Diamond_s\chi_k$. Note that since each conjunct we apply this rule to is not in $\KLMLang$ it must be bounded by a standpoint modality. Therefore, we are able to reverse these steps by reducing $\Diamond_s\chi_k$ to $\chi_k$ using similar steps in point 2. of the proof. Then we are able to apply \textbf{RN} to obtain $\Box_s\chi_k$ and finally apply \textbf{K.'} to add $\chi_k$ back to the diamond-bound conjunction.
             \end{itemize}
     \end{enumerate}
     \item[D.] We repeat the algorithm on each conjunct as described until we only have conjuncts which are not unaltered by the rules given in the algorithm.
 \end{itemize}

 Clearly, this process terminates in finitely many steps since $\phi$ is of finite length. In fact it is computable in polynomially many steps: in the whole formula, the number of times we apply \textbf{4.b.} \textbf{5.b.} \textbf{4'.}, \textbf{5'} and \textbf{D.} is bounded by the number of modalities that occur in the formula. Moreover, the number of times we apply the $\Box$\textbf{-Dist.} and  $\Diamond$\textbf{-Dist.} rules is bounded above by the number of times the conjunction symbol ``$\wedge$'' occurrs in the formula. Hence, we will only apply a polynomial number of ruels in the size of the formula to find its equivalent normal form.
  
 Furthermore, the process described above will repeat until we obtain a set of formulas which are either members of $\KLMLang$ or of the form $\#_s\zeta$ where $\zeta\in\KLMLang$. To see that is the case, assume that some formula in the set of conjuncts is not in this form: then by the description above we must either apply step B.2. or B.3. to this and hence alter the formulas. Therefore, the above algorithm only terminates when our conjunct is in normal form. Then, we take the conjunction of all the formulas which we obtain from reducing conjuncts, and we obtain $\phi'$ in normal form. Moreover, since we can reverse each step here, if a set $\A$ is preferentially closed then $\phi\in\A$ iff $\phi'\in\A$.
\end{proof}

\textbf{Lemma \ref{lemma:derived-KLMprop-sets-of-closed-set-are-closed}.}\textit{ If $\A$ is preferentially closed, then each derived set of the form $\As$ is closed under the KLM postulates and classical propositional logic. Furthermore, for all $s\in \mathcal{S}$ and all $\Diamond \psi \in A$ where $\psi\in \KLMLang$ we have that each set in $\textbf{MaxConj.}_\A(\psi,s)$ is closed under the KLM postulates and classical propositional logic.}

\begin{proof}
    We first consider the KLM postulates for sets of the form $\As$. First note that by reflexivity, $\alpha\twiddle\alpha\in \A$ for any Boolean formula $\alpha$. Then, by \textbf{RN} $\Box_s(\alpha\twiddle\alpha)\in\A$ and thus $\alpha\twiddle\alpha\in \A_s$. 

    The other KLM style rules all follow a similar argument, so we will restrict ourselves to proving that $\As$ is closed under \textbf{LLE}. Assume that $\alpha\leftrightarrow\beta, \alpha\twiddle\gamma\in\As$. Then $\Box_s(\alpha\leftrightarrow\beta), \Box_s(\alpha\twiddle\gamma)\in\A$. Then by \textbf{LLE.a.} we must have $\Box_s(\beta\twiddle\gamma)\in\A$. Therefore, $\beta\twiddle\gamma\in\As$ and so $\As$ is closed under \textbf{LLE}. The other KLM postulates follow by a similar argument and so $\As$ is closed under the KLM postulates.

    In the case of $\Asipsi\in \textbf{MaxConj.}_\A(\psi,s)$, assume that $\alpha\twiddle\alpha\notin\Asipsi$. Then it must be the case that there exists an infinite chain of conjuncts $\omega=(C_i)_{i\in \mathbb{N}}$ in $\textbf{Conj}_\A(s,\psi)$ such that $\alpha\twiddle\alpha$ is not in the union of this chain. However, for any conjunction of the form $\Diamond_s(\psi\wedge\Gamma)\in \A$, since $\Box_s(\alpha\twiddle\alpha)\in \A$, by \textbf{K.b.} we must have $\Diamond_s((\alpha\twiddle\alpha)\wedge\psi\wedge\Gamma)\in \A$. Therefore if $\Asipsi=\bigcup^\infty_{i=1}C_i$ we must have that the sequence $\omega'=(C_i\cup\{\alpha\twiddle\alpha\})$ is an infinite chain of conjuncts in $\textbf{Conj}_\A(s,\psi)$. Furthermore, $\A^\psi_{s,\omega'}:=\bigcup^\infty_{i=1}(C_i\cup\{\alpha\twiddle\alpha\})$ is such that $\A^\psi_{s,\omega'}\in \textbf{LimConj}_\A(\psi,s)$ and $\Asipsi\subset \A^\psi_{s,\omega'}$. However, this contradicts the subset-maximality of $\Asipsi$. Therefore, if $\Asipsi\in \textbf{MaxConj.}_\A(\psi,s)$ then $\alpha\twiddle\alpha\in \Asipsi$ and so $\Asipsi$ is closed under reflexivity.

As in the previous case, a similar method can be used for the rest of the postulates and so we consider only LLE in this proof. Assume that $\alpha\leftrightarrow\beta,\alpha\twiddle\gamma\in\Asipsi$. Then there exists some conjunction $\Gamma$ such that $\Diamond_s((\alpha\leftrightarrow\beta)\wedge(\alpha\twiddle\gamma)\wedge\psi\wedge \Gamma)\in\A$. However, for any such conjunction we must have $\Diamond_s((\beta\twiddle\gamma)\wedge (\alpha\leftrightarrow\beta)\wedge(\alpha\twiddle\gamma)\wedge\psi\wedge\Gamma)\in\A$ by \textbf{LLE.b.}. Therefore, the union of any infinite chain of conjuncts containing $\psi, \alpha\leftrightarrow\beta$ and $\alpha\twiddle \gamma$ must either eventually contain $\beta\twiddle \gamma$, or be non-maximal\footnote{Since as in the proof of reflexivity, if $\Asipsi$ is non-maximal, we can find another infinite chain whose union contains $\Asipsi$ and $\beta\twiddle \gamma$.}. And so if $\alpha\leftrightarrow\beta,\alpha\twiddle\gamma\in\Asipsi$, we must have $\beta\twiddle\gamma\in \Asipsi$. By a similar argument the other postulates hold and so $\Asipsi$ is closed under the KLM postulates.

To show such sets are closed under propositional logic, we prove the sufficient that condition thay any such sets contain all Boolean tautologies and are closed under modus ponens. Since $\A$ is preferentially closed any Boolean tautology $\phi$ is included in $\A$. Therefore, by \textbf{RN} we have that $\Box_s\phi\in \A$ and hence $\phi \in \As$ for all Boolean tautologies $\phi$. Furthermore, by \textbf{K.b.} and the fact that $\Box_s\phi\in \A$, for any formula $\Diamond_s\Gamma\in \A$ where $\Gamma\in \KLMLang$ we must have that $\Diamond_s(\Gamma\wedge \phi)\in \A$ and so the limit of any infinite chain of conjuncts in $\textbf{Conj}_\A(s,\psi)$ either contains $\phi$ or is non-maximal. Hence any $\Asipsi \in \textbf{MaxConj}_\A(s,\psi)$ contains all Boolean tautologies. 

To see that any $\Asipsi \in \textbf{MaxConj}_\A(s,\psi)$ is closed under modus ponens we note that \textbf{K.c.} entails that any infinite chain of conjuncts in $\textbf{Conj}_\A(s,\psi)$ either satisfies modus ponens or is non-maximal. Similarly, \textbf{K.a.} entails that each set of the form $\As$ satisfies modus ponens.

\end{proof}

\textbf{Lemma \ref{lemma:if-A-satisfiable-then-derived-KLM-sets-satisfiable}}\textit{ If $\A$ is satisfiable, then $\As$ and $\Asipsi\in \textbf{MaxConj.}_\A(\psi,s)$ are non-trivially satisfiable.}

\begin{proof}
    If $\A$ is satisfiable, then there exists some preferential standpoint structure $M=(\Pi,\sigma,\tau)$ such that $M\Vdash\phi$ for all $\phi\in\A$. Then, in particular $M\Vdash \Box_s\xi$ for all formulas in $\A$ in of the form $\Box_s\xi$ where $\xi\in\KLMLang$. Then since $\sigma(s)\neq \emptyset$ by definition there exists some $\pi\in\sigma(s)$ such that $\tau(\pi)\Vdash \xi$ for all $\xi\in\{\xi\in\KLMLang\mid \Box_s\xi\in\A\}=\As$. Thus, $\As$ is satisfiable for all $s\in\mathcal{S}$. In both cases, the satisfiability of $\A$ implies that $\tau(\pi)$ is non-empty for each $\pi\in \Pi$, and so all sets of the form $\As$ are non-trivially satisfiable.

    Then for any $\Diamond_s\psi\in\A$ and any $\omega\in\textbf{MaxConj}_\A(s,\psi)$ we consider the set $\Asipsi$. By compactness of original KLM logic \cite{kraus:nonmonotonic} we have that there is some finite subset $X\subseteq \Asipsi$ such that if $\Asipsi$ preferentially entails $\xi$, then $X$ preferentially entails $\xi$. By Theorem \ref{theorem:KLM-original-pref-rep-result} it follows  that for any preferential interpretation $\I$, we have that $\I\Vdash X$ implies $\I\Vdash \Asipsi$. Then, since $X$ is a finite subset of $\Asipsi$, by definition there must exist some $C_i$ in the sequence $\omega$ such that $X\subseteq C_i$. But then $\Diamond_s(\psi\wedge (\bigwedge X)\wedge \Gamma)\in \A$ for some finite conjunction $\Gamma$. Since $\A$ is satisfiable then $M\Vdash \Diamond_s(\psi\wedge (\bigwedge X)\wedge \Gamma)$ and so there exists some $\pi\in\sigma(s)$ such that $\tau(\pi)\Vdash \psi\wedge (\bigwedge X)\wedge \Gamma$. In particular, $\tau(\pi)\Vdash x$ for all $x\in X$ and so by compactness $\tau(\pi)\Vdash \xi$ for all $\xi\in\Asipsi$. Therefore $\Asipsi$ is satisfiable.
\end{proof}

\textbf{Lemma \ref{lemma:M_A-constructed-is-the-model-that-represents-A}.}\textit{ For any satisfiable, preferentially closed set $\A$, we have that $\phi\in\A$ iff $M_\A\Vdash \phi$.}

\begin{proof}
    Due to the result of Lemma \ref{lemma:proof-theoretic-normal-form}, we only consider standpoint sharpenings and formulas in normal form and consider five cases for each implication. We use $M$ as shorthand for $M_\A$ in this case.\\

    \textit{Case 1:} Suppose $\phi\in \A$. In our base case, if $\phi\in\KLMLang$ then by \textbf{RN} we have that $\Box_s\phi\in\A$ for every $s\in\mathcal{S}$. Therefore $\phi\in\As$ and so $\Is\Vdash \phi$, or equivalently $M,\pi_s\Vdash \phi$. We also have that for any formula of the form $\Diamond_s(\psi\wedge \Gamma)\in\A$, by \textbf{K'} we have that $\Diamond_s(\phi\wedge\psi\wedge \Gamma)\in\A$. Therefore, for the union of any infinite chain of conjuncts $\Asipsi$, we must eventually have that $\phi\in \Asipsi$. Then $\Isipsi\Vdash \phi$ and so $M,\pi^
    \psi_{s,\omega}\Vdash \phi$. That is $M,\pi\Vdash \phi$ for all $\pi\in \Pi$ and so $M\Vdash \phi$.\\

    \textit{Case 2:} For a formula of the form $\Box_s\phi\in \A$ where $\phi\in\KLMLang$ we note the following. For any $t\in\mathcal{S}$ such that $t\preceq s\in\A$ we must have that $\Box_t\phi\in \A$ by \textbf{$P.a.$}. Then for any such $t$ we have $\phi\in\A_t$ and so $M,\pi_t\Vdash \phi$.

    We also obtain that for any $t\in\{t\in\mathcal{S}\mid t\preceq s\in\A\}$ and any formula of the form $\Diamond_t(\psi\wedge \Gamma)\in \A$, by \textbf{K'} we have $\Diamond_t(\phi\wedge \psi\wedge \Gamma)\in\A$. Because of this, for any infinite chain in $\textbf{Conj}_\A(\psi,t)$ we must eventually include $\phi$ and so $\phi\in\A^\psi_{t,\omega}$  for all such sets where $t\preceq s\in \A$. Equivalently $M,\pi^
    \psi_{s,\omega}\Vdash \phi$ for all such $t$. Hence, $M,\pi\Vdash \phi$ for all $\pi\in\sigma(s)$ and so $M\Vdash \Box_s\phi$. In the special case where $s=*$ we only need to add that $t\preceq *\in\A$ for every $t\in\mathcal{S}$ to obtain that $M,\pi\Vdash \phi$ for every $\pi\in\Pi$.\\

    \textit{Case 3:} Consider a formula of the form $\Diamond_s\phi\in \A$ where $\phi\in\KLMLang$. Then by definition $\phi$ is in every conjunct in $\textbf{Conj}_\A(\phi,s)$ and so for any sequence $\omega$ we have that $\phi\in\A^\phi_{s,\omega}$ and therefore $M,\pi^\phi_{s,\omega}\Vdash \phi$. Lastly note that since $s\preceq s\in\A$ by \textbf{$\preceq$-Ref.} we have that $\pi^\phi_{s,\omega}\in\sigma(s)$ and so $M\Vdash \Diamond_s\phi$.\\

    \textit{Case 4:} If $\phi\in\A$ is a conjunction of cases 1.-3. then we know that $\phi\in A$ iff each of its conjuncts is in $\A$. Furthermore, by cases 1.-3. we know $M$ satisfies each conjunct and therefore satisfies the whole conjunction. That is, $M\Vdash \phi$.\\

    \textit{Case 5:} If $\phi=s\preceq t\in \A$ then suppose $M\nVdash s\preceq t$. Then there is some $\pi\in \sigma(s)$ such that $\pi\notin \sigma(t)$. Since $\pi\in\sigma(s)$ then either $\pi=\pi_{t'}$ where $t'\preceq s\in \A$ or $\pi=\pi^\psi_{t',\omega}$ where $t'\preceq s\in \A$. In either case, by \textbf{$\preceq$-Trans} we have that $t'\preceq t$ and so by definition $\pi\in \sigma(t)$ which is a contradiction. Hence $M\Vdash s\preceq t$.\\ 

Due to Lemma \ref{lemma:proof-theoretic-normal-form} this is sufficient to show that $\phi\in\A$ implies $M\Vdash \phi$ for all $\phi\in\KLMLang$. We now show the converse holds using similar cases:\\ 

\textit{Case 1:} Assume $M\Vdash \phi$ for $\phi\in\KLMLang$. Then $M,\pi\Vdash \phi$ for all $\pi\in \Pi$. In particular, $M,\pi_*\Vdash \phi$ and so $\I_*\Vdash \phi$, and so we have $\phi\in \A_*$. Then $\Box_*\phi\in \A$ and by \textbf{T$^*$} $\phi\in\A$.\\

\textit{Case 2:} Assume $M\Vdash \Box_s\phi$ for $\phi\in\KLMLang$. Then $M,\pi\Vdash \phi$ for all $\pi\in \sigma(s)$. In particular, $M,\pi_s\Vdash \phi$ and $\I_s\Vdash \phi$. So by definition $\phi\in \As$ and $\Box_s\phi\in\A$.

\textit{Case 3:} Assume $M\Vdash \Diamond_s\phi$ for $\phi\in\KLMLang$. Then $M,\pi\Vdash\phi$ for some $\pi\in \sigma(s)$. Either $\pi=\pi_t$ where $t\preceq s\in \A$ or $\pi=\pi^\psi_{t,\omega}$ where  $t\preceq s\in \A$, $\Diamond_t\psi\in \A$ and $\omega\in \textbf{MaxConj}_\A(t,\psi)$. In the first case we have that $\I_t\Vdash \phi$ and so $\Box_t\phi\in\A$. Then, by \textbf{D} $\Diamond_t\phi\in \A$ and since $t\preceq s\in \A$ by \textbf{P.b} we obtain $\Diamond_s\phi\in \A$. In the second case $\I^\psi_{t,\omega}\Vdash \phi$ and so $\phi\in \A^\psi_{t,\omega}$. Then by definition $\Diamond_t(\phi\wedge\psi\wedge \Gamma)\in \A$. By \textbf{$\Diamond$-Dist} we have $\Diamond_t\phi\wedge\Diamond_t(\psi\wedge \Gamma)\in \A$ and by classical logic rules $\Diamond_t \phi\in \A$. Then similarly by \textbf{P.b} we obtain $\Diamond_s\phi\in \A$.\\

\textit{Case 4:} $M\Vdash \phi$ where $\phi$ is a conjunction of cases 1.-3. Then similarly we use the fact that $M$ must satisfy each conjunct, meaning that each conjunct is a member of $\A$, and therefore so is the whole conjunction. Therefore, $\phi\in \A$.\\

\textit{Case 5:} $M\Vdash s\preceq t$. If $t=*$ or $t=s$ then this is an axiom and already appears in $\A$. Otherwise 
assume $s\preceq t\notin \A$. Then $\pi_s\notin \sigma(t)$ since $\sigma(t)=\{\pi_{t'}\mid t'\preceq t\in\mathcal{A}\}\cup \{\pi_{t',i}^\psi\mid t'\preceq s\in \mathcal{A}\}$. But since $M\Vdash s\preceq t$, we have $\sigma(s)\subseteq \sigma(t)$. But clearly $\pi_s\in \sigma(s)$ since $s\preceq s\in \A$. This is a contradiction and therefore $s\preceq t\in \A$.
\end{proof}

\section{Proofs of Results in Section \ref{section:entailments}}

\subsection{Section \ref{subsection:preferential-entailment}}

\textbf{Corollary \ref{corollary:preferential-entailment-is-smallest-closed-set}.}\textit{ $\phi\in \DRSLLang$ is preferentially entailed by $\KB$ iff $\phi\in C(\KB)$, where $C(\KB)$ is the set containing obtained by exhaustively applying the rules in Figures \ref{fig:DRSL-modal-rules} and \ref{fig:DRSL-KLM-postulates} to $\KB$.}

\begin{proof}
  By definition $C(\KB)$ is preferentially closed, and by Theorem \ref{theorem:main-representation-result}, there exists some preferential model $M^C$ such that $M^C\Vdash\phi$ iff $\phi\in C(\KB)$. Then note that for any other preferentially closed set $\A$ such that $\A\in \KB$, we must have that $C(\KB)\subseteq \A$. But then, for any other $M$ such that $M\Vdash \KB$, we have that $M\Vdash \phi$ for all $\phi\in C(\KB)$. Hence, $\KB\vDash_P\phi$ for all $\phi\in C(\KB)$. On the other hand, the existence of $M^C$ implies that there cannot exist $\phi$ such that $\KB\vDash_P\phi$ and $\phi\notin C(\KB)$. That is $\KB\vDash_P\phi$ iff $\phi\in C(\KB)$.
\end{proof}

\textbf{Proposition \ref{proposition:pref-entailment-of-box-statements-reducible}}\textit{ Consider a DRSL knowledge base $\KB\subseteq\DRSLLang$, and a statement $\Box_s\psi\in \DRSLLang$, where $\psi\in\KLMLang$. Then, $\KB\vDash_P\Box_s\psi$ iff $\KB_s\vDash_{P,\textit{prop}} \psi$.}

\begin{proof}
$\Rightarrow$: In order to obtain a contradiction, suppose $\KB\vDash_P \Box_s\psi$ and $\KB_s\nvDash_{P,\textit{prop}} \psi$. Then, there exists some preferential interpretation $\I^*$ such that $\I\Vdash \KB_s$ and $\I^*\nVdash \psi$. Then suppose $M=(\Pi,\sigma,\tau)$ is a model of $\KB$. That is, $M\Vdash \KB$. We extend $M$ to $M'=(\Pi\cup\{\pi^*\},\sigma',\tau')$, where $\sigma'(s)=\sigma(s)\cup\{\pi^*\}$ and for all other standpoints $t$ we define $\sigma'(t)=\sigma(t)\cup\{\pi^*\}$ if $s\preceq t\in C(\KB)$, $\sigma'(t)=\sigma(t)$ otherwise. $\tau'(\pi^*)=\I^*$ and $\tau'(\pi)=\tau(\pi)$ for all other precisifications. 

We show here that $M'$ is still a model of $\KB$. Firstly not that $\pi^*$ has been added to $M$ specifically in a way that preserves the standpoint sharpening rules in $\KB$. Furthermore, since all previous precisifications and their images under $\tau$ still occur in $M'$, then diamond bound statements are not affected. That is, if $M\Vdash \Diamond_t \xi$, then $M\Vdash \Diamond_t \xi$. For any box bound statement in $\Box_t\xi\in\KB$ we have that either $\sigma'(t)=\sigma(t)$ in which case $M\Vdash \Box_t\xi$ implies that  $M'\Vdash \Box_t\xi$, or $\sigma'(t)=\sigma(t)\cup\{\pi^*\}$ and $s\preceq t\in C(\KB)$. By definition $M',\pi\Vdash \xi$ for all $\pi\in \sigma(t)$. Then note that since $\I^*$ is a model of $\KB_s$, then $\I^*\Vdash \phi$ whenever $\Box_t\phi\in \KB$ and $s\preceq t\in C(\KB)$. Therefore $M',\pi^*\Vdash \xi$. Under the assumption that $\KB$ is in conjunction-free normal form, this shows that $M'$ is a model of $\KB$.

Then $M'\nVdash \Box_s\psi$ since $M',\pi^*\nVdash \psi$. But then, since $M'$
 is a model of $\KB$, this implies that $\KB\nvDash_P\Box_s\psi$ which is a contradiction.\\

 $\Leftarrow:$ If $\KB_s\vDash_{P,\textit{prop}} \psi$, then for any preferential standpoint structure $M$ such that $M\Vdash \KB$, note the following. For any $\Box_t\xi \in \KB$ such that $s\preceq t\in C(\KB)$ (i.e., for any $\xi\in \KB_s$ we have that $M,\pi\Vdash \xi$ for all $\pi\in \sigma(s)$. That is $\tau(\pi)\Vdash \xi$ for all $\xi\in \KB_s$. But then, since $\tau(\pi)$ is a model of $\KB_s$, by assumption $\tau(\pi)\Vdash \psi$. Hence, $M,\pi\Vdash \psi$ for all $\pi\in \sigma(s)$ and so $M\Vdash \Box_s\psi$. Since this holds for all models of $\KB$, we have $\KB\vDash_P\Box_s\psi$.
 \end{proof}

\textbf{Proposition \ref{proposition:preferential-entailment-diamond-statements}}\textit{ Consider a DRSL knowledge base $\KB\subseteq\DRSLLang$, and a statement $\Diamond_s\psi\in \DRSLLang$, where $\psi\in\KLMLang$. Then $\KB\vDash_P \Diamond_s\psi$ iff $X\vDash_{P,\text{prop}} \psi$ for some $X\in \textbf{PropKB}_{\KB}(s)$.}

 \begin{proof}
    $\Rightarrow:$ Suppose $\KB\vDash_P \Diamond_s\psi$ and assume there is no $\KB^*\in \textbf{PropKB}_{\KB}(s)$,  such that $\KB^*\vDash_{P,\text{prop}}\psi$. Then, to obtain a contradiction, we build the following model $M=(\Pi,\sigma,\tau)$ of $\KB$:

    \begin{itemize}
        \item $\Pi=\{\pi_s\mid s\in \mathcal{S}\}\cup\{\pi^\phi_s\mid \Diamond_s\phi\in\KB\}$.
        \item $\sigma(t)=\{\pi_u\mid u\preceq t\in C(\KB)\}\cup \{\pi^\phi_u\mid u\preceq t\in C(\KB)\}$.
        \item $\tau(\pi_t)=\I_t$ and $\tau(\pi^\phi_t)=\I^\phi_t$
    \end{itemize}

    where $\I_t$ is a model of $\KB_t$ and $\I^\phi_t$ is a model of $\KB^\phi_t$. In particular we choose these models such that  $\I_t\nVdash \psi$ and $\I^\phi_t\nVdash \psi$ whenever $\KB_t$ or $\KB^\phi_t$ are in $\textbf{PropKB}_{\KB}(s)$. By our assumption such models must exist.

    In order to see that this is a model of $\KB$ note the following: For each $\Diamond_t\phi\in\KB$ we have that $\tau(\pi^\phi_t)\Vdash \phi$ by definition. Then, since $t\preceq t\in C\KB)$ we have that $\pi^\phi_t\in \sigma(t)$ and so $M\Vdash \diamond_t\phi$. For each $\Box_t\xi\in \KB$, we have that $\xi\in \KB_u$ for any standpoint $u$ such that $u\preceq t\in C(\KB)$, and therefore $M,\pi_u\Vdash \xi$. Moreover, for any such $u$ have that $\xi\in \KB^\phi_u$ for each formula of the for $\Diamond_u\phi\in \KB$, and therefore $M,\pi^\phi_u\Vdash \xi$. Therefore $M,\pi\Vdash \xi$ for all $\pi\in \{\pi_u\mid u\preceq t\in C(\KB)\}\cup \{\pi^\phi_u\mid u\preceq t\in C(\KB)\}=\sigma(t)$. Hence, $M\Vdash \Box_t \xi$. Lastly, for any $t\preceq u\in \KB$, notice that $\sigma$ is defined such that $\sigma(t)\subseteq \sigma(u)$. Since we assume $\KB$ is in conjunction free normal form, this is sufficient to show that $M$ is a model of $\KB$.

    Then note that $\sigma(s)=\{\pi_t\mid t\preceq s\in C(\KB)\}\cup \{\pi^\phi_t\mid t\preceq s\in C(\KB)\}$. However, by definition for each $\pi_t\in \sigma(s)$ we have that $\tau(\pi_t)=I_t$ is a model of $\KB_t$ where $\KB_t\in \textbf{PropKB}_{\KB}(s)$, and once again by construction $\tau(\pi_t)=\I_t\nVdash \psi$. For similar reasons $\tau(\pi^\phi_t)\nVdash\psi$ for all $\pi^\phi_t\in \sigma(s)$. That is, $M,\pi\nVdash \psi$ for all $\pi\in \sigma(s)$, and so $M\nVdash \Diamond_s\psi$. This contradicts our assumption that $\KB\vDash_P \Diamond_s\psi$.\\

    $\Leftarrow:$ Suppose $\KB^*\vDash_{P,\text{prop}}\psi$ for some $\KB^*\in \textbf{PropKB}_{\KB}(s)$. Then we consider two cases:

    \begin{itemize}
        \item \textit{Case 1:} If $\KB^*=\KB_t$ for some standpoint $t$ such that $t\preceq s\in C(\KB)$ then by our previous Proposition \ref{proposition:pref-entailment-of-box-statements-reducible} we must have that $\KB\vDash \Box_t \psi$ and by an application of rules \textbf{D.} and \textbf{P.b.} then $\KB\vDash \Diamond_s \psi$.
        \item \textit{Case 2:} If $\KB^*=\KB^*_t$ for standpoint $t$ such that $t\preceq s\in C(\KB)$ then for any model $M$ of $\KB$ we have that since $\Diamond_t\phi\in \KB$ that $M\Vdash \Diamond_t\phi$ and so $M,\pi\Vdash \phi$ for some $\pi^*\in \sigma(t)$. Moreover, for any $\Box_t\xi\in \KB$ we have that $M,\pi^*\Vdash \xi$. But then $\tau(\pi^*)\Vdash \KB_t\cup \{\phi\}=\KB^\phi_t$ and so $\tau(\pi^*)$ is a model of $\KB^*$. But then by assumption, $\tau(\pi^*)\Vdash \psi$. Lastly, note that since $t\preceq s\in C(\KB)$ then $\sigma(t)\subseteq \sigma(s)$ and in particular $\pi^*\in \sigma(s)$. Hence, $M\Vdash \Diamond_s \psi$.
    \end{itemize}

    Thus, we have shown that if $\KB^*\vDash_{P,\text{prop}}\psi$ for some $\KB^*\in \textbf{PropKB}_{\KB}(s)$, then any model of $\KB$ satisfies $\Diamond_s\psi$. Equivalently, $\KB\vDash_P\Diamond_s \psi$.
\end{proof}

    \textbf{Theorem \ref{theorem:complexity-of-DRSL-preferentail-entailment}.}\textit{ Preferential entailment-checking in DRSL is \textsc{CoNP}-complete.}

\begin{proof}
    \textsc{CoNP}-hardness follows from the fact that preferential entailment is  \textsc{CoNP}-complete, and that DRSL preferential entailment contains the propositional case.

    For \textsc{CoNP} membership note the following. Assume we are given a DRSL knowledge base $\KB$, and a query $\xi$. If $\xi$ is a standpoint sharpening we then entailment checking reduces to computing transitive closure, which is in \textsc{P} time. If $\xi$ is not a standpoint sharpening we can assume it is in normal form (and if it is not in normal form we can reduce it to normal form in polynomial time). Then $\xi$ is a conjunction of formulas of the form $\phi$, where either $\phi=\psi$, $\phi=\Box_s\psi$ or $\phi=\Diamond_s\psi$ for $\psi\in\KLMLang$.

    If $\phi=\psi$ then we can check entailment for this conjunct by checking whether $\KB_*\vDash_{P,prop} \psi$ by Proposition \ref{proposition:pref-entailment-of-box-statements-reducible}.  

    Similarly, if $\phi=\Box_s\psi$ then we can check entailment for this conjunct by checking whether $\KB_s\vDash_{P,prop} \psi$.

    If $\phi=\Diamond_s\psi$ then we can check entailment by checking whether $X\vDash_{P,prop} \psi$ for some $X\in \textbf{PropKB}_\KB(s)$.

    We here note that  $\textbf{PropKB}_\KB(*)=\bigcup_{s\in\mathcal{S}}\textbf{PropKB}_\KB(s)$, and furthermore we note that the size $\textbf{PropKB}_\KB(*)$ is in the worst case the size of $\KB$, since at most each element of $\KB$ initiates the need for a new member of $\textbf{PropKB}_\KB(*)$. Hence for any $s$, $|\textbf{PropKB}_\KB(s)|\leq |\textbf{PropKB}_\KB(*)|\leq |\KB|$. Moreover, each element of any $X\in \textbf{PropKB}_\KB(*)$ is a subformula of an element of $\KB$, and so the size of $X$ is no larger than $\KB$.

    Therefore, we can reduce checking whether $\KB\vDash \xi$ to polynomially many propositional preferential entailment checks (in the size of $\KB$ and $\xi$), where each check is polynomial in the size of $\KB$ and $\xi$.

    Now note the following. Since none of the preferential entailment checks are dependent on each other, we can replace the atoms occurring in all the formulas involved in the preferential entailment checks. That is, for each conjunct $c$ in $\xi$ and each $X\in \textbf{PropKB}_\KB(*)$ for which a propositional preferential entailment check is required, we can replace each $p\in\mathcal{P}$ with the indexed atom $p_c^X$ which results in a logically equivalent entailment check, but where each check has a distinct vocabulary.

     Let $C$ be the set of conjuncts in $\xi$ and let $\textbf{PropKB}_\KB(c)$ denote the subset of $\textbf{PropKB}_\KB(*)$ used when querying preferential entailment for $c$. Then, since each propositional preferential entailment check is \textsc{CoNP}-complete we can polynomially reduce it to an \textsc{UNSAT} query of a Boolean formula $q_c^X$ for each $c\in C$ and $X\in \textbf{PropKB}_\KB(c)$. This Boolean formula is polynomial in the size of $\KB$ and $\xi$. That is, checking if $\KB\vDash_P \xi$ is equivalent to checking whether $q$ is unsatisfiable for each $q\in \{q_c^X\mid c\in C, X\in \textbf{PropKB}_\KB(c)\}$.
     
     Since we are able to query each preferential entailment with a distinct vocabulary, we can assume w.l.o.g. that the atoms occurring in each $q_c^X$ are distinct. Since our vocabulary is partitioned for each conjunct, we have that each $q\in \{q_c^X\mid c\in C, X\in \textbf{PropKB}_\KB(c)\}$ is unsatisfiable if and only if the conjunction 
    \[\bigwedge_{c\in C,X\in \textbf{PropKB}_\KB(c)}q_c^X\]
    is unsatisfiable. Then, since the number of such conjunctions are polynomial the size of $\KB$ and $\xi$, and each $q_c^X$ is also of polynomial size, we can reduce checking is $\KB\vDash \xi$ holds to a single polynomially large \textsc{UNSAT} check, and hence the complexity is dominated by the \textsc{CoNP} complexity of \textsc{UNSAT}. Therefore checking preferential entailment in DRSL is in \textsc{CoNP}.
\end{proof}

\subsection{Section \ref{subsection:ranked-entailments}}

\textbf{Proposition \ref{proposition:ranked-entailment-preserves-satisfiability-and-contains-the-KB}. }\textit{Consider $\KB\subseteq \DRSLLang$, a defeasible entailment $\dentails$ and the associated ranking strategy $r$:
\begin{enumerate}
    \item If $\KB$ is satisfiable, then there is no $\pi_X\in \Pi_{\KB}$ such that $\tau_{\KB}^r(\pi_X)=\emptyset$.
    \item If $\phi\in\KB$ then $M_{\KB}^{\dentails}\Vdash \phi$.
\end{enumerate}}

\begin{proof}
\begin{enumerate}
    \item We show this by contrapositive. For $\KB\subseteq\DRSLLang$ suppose there exists an defeasible entailment $\dentails$ with a ranking strategy $r$ where $M_\KB^{\dentails}=(\Pi_\KB,\sigma_\KB,\tau_\KB^r)$ and $\tau_\KB^r(\pi_X)=\emptyset$ for some $\pi_X\in \Pi_\KB$. We show that in this case $\KB$ must be unsatisfiable.

    Then $\tau_\KB^r(\pi_X)=r_X=\emptyset$ for some $X\in \textbf{PropKB}_\KB(*)$. Moreover, since $r_X$ is empty, we have that $r_X\Vdash \phi$ for all $\phi\in\KLMLang$. In particular, $r_X\Vdash \bot$. But then, by construction, we have that $X\dentails \bot$. Furthermore, since we assume that $\dentails$ satisfies classical preservation, $X\dentails \bot$ iff $X\vDash_{P,prop} \bot$. However, then since $X\in \textbf{PropKB}_\KB(*)$, we have by Proposition \ref{proposition:preferential-entailment-diamond-statements} that $\KB\vDash_P\Diamond_*\bot$. But this is true iff in every model $M=(\Pi,\sigma,\tau)$ of $\KB$ we have that $M,\pi\Vdash \bot$ for some $\pi\in \Pi$. That is, $\tau(\pi)\Vdash \bot$. Since the only preferential interpretation that satisfies $\bot$ is the empty interpretation, we then have that $\tau(\pi)=\emptyset$. And therefore, $\KB$ is not satisfiable.

    \item Let $\phi\in \KB$ we show by cases on the structure of $\phi$ that $\KB\dentails \phi$.
    \begin{itemize}
        \item If $\phi=s\preceq t$ then by definition $\textbf{PropKB}_\KB(s)\subseteq \textbf{PropKB}_\KB(t)$ and so if $\pi_X\in \sigma_\KB(s)$, then $X\in \textbf{PropKB}_\KB(s)$. Thus  $X\in \textbf{PropKB}_\KB(t)$ and so by definition $\pi_X\in \sigma(t)$. Hence $M_\KB^{\dentails}\Vdash s\preceq t$.
        \item If $\phi\in \KLMLang$ then $\phi\in X$ for all $X\in \textbf{PropKB}_\KB(*)$ and therefore, for each $\pi\in \Pi_\KB$ we have $\tau_\KB^r(\pi)=r_X$ for some $X\in \textbf{PropKB}_\KB(*)$. By our assumption that $\dentails$ satisfies \textit{Inclusion} we have that for each $X\in \textbf{PropKB}_\KB(*)$ that $X\dentails \phi$ and therefore $r_X\Vdash \phi$. That is $M^{\dentails}_\KB,\pi\Vdash \phi$ for all $\pi\in \Pi$ and so $M^{\dentails}_\KB\Vdash \phi$.
        \item If $\phi=\Box_s \psi$ for $\psi\in \KLMLang$ then using similar reasoning as in the previous case we have $\psi\in X$ for all $X\in \textbf{PropKB}_\KB(s)$ and so by a similar argument $M,\pi\Vdash \psi$ for all $\pi\in \sigma_\KB(s)$.
    \end{itemize}
    Since $\KB$ is in conjunction-free normal form, this is sufficient to cover all cases.
\end{enumerate}    
\end{proof}

\textbf{Proposition \ref{proposition:defeasible-entailment-collapses-into-propositional-case-open-world}. }\textit{Suppose $\KB=\KB'\cup\{\Diamond_*\top\}$ where $\KB'\subseteq\KLMLang$. Then for $\phi\in \KLMLang$, we have $\KB\dentails_{DRSL} \phi$  iff $\KB'\dentails_{prop} \phi$.}

\begin{proof}
    In the case specified above, it is clear that $\textbf{PropKB}_\KB(*)=\{\KB_*,\KB_*^\top\}$ and so in this case if $M^{\dentails}_\KB=(\Pi_\KB,\sigma_\KB,\tau_\KB^r)$ we have that $\Pi_\KB=\{\pi_{\KB_*^\top}\}$ and $\tau_\KB^r(\pi_{\KB_*^\top})=r_{\KB_*^\top}$. Therefore $M^{\dentails}_\KB\Vdash \phi$ for $\phi\in \KLMLang$ iff $r_{\KB_*^\top}\Vdash \phi$. Furthermore, by definition $r_{\KB_*^\top}\Vdash \phi$ $\KB^\top_*\dentails \phi$. Then note that since $\KB^\top_*=\KB_*\cup\{\top\}$ and the addition of $\top$ is tautologous we have $\KB^\top_*\dentails \phi$ iff $\KB_*\dentails \phi$. Finally, since $\KB'\subseteq \KLMLang$ we have that $\KB'=\KB_*$ and so $\KB_*\dentails \phi$ iff $\KB'\dentails_{prop} \phi$.
\end{proof}

 \textbf{Theorem \ref{theorem:algorithmic-representation-result}.}\textit{ For a given DRSL knowledge base $\KB$, a formula $\phi\in\DRSLLang$ and a defeasible entailment $\dentails$ with a selection strategy $r$, we have that $\KB\dentails \psi$ iff $StdptRankEntail(\KB,r,\phi)=\text{True}$.}

\begin{proof}
    We once again consider the base case for single box and diamond quantified statements, and then prove by induction over conjunctions:

    \begin{itemize}
        \item If $\phi=\Box_s\psi$ for $\psi\in\KLMLang$ then $\KB\dentails \phi$ iff $M_\KB^{\dentails},\pi_X\Vdash \psi$ for all $\pi_X\in \sigma(s)$ iff
        $\tau_\KB^r(\pi_X)=r_X\Vdash \phi$ for all $X\in \textbf{PropKB}(s)\setminus\{\KB_*\}$. By construction this is equivalent to $X\dentails \psi$ for all $X\in \textbf{PropKB}(s)\setminus\{\KB_*\}$. Lastly, by the result given by Casini, Meyer and Varzinczak \cite{casini:beyondratclosure}, this is true iff $\text{DefeasibleEntail}(X,r_X,\psi)=True$ for all $X\in \textbf{PropKB}(s)\setminus\{\KB_*\}$. Equivalently, $\KB\ndentails \phi$ iff $\text{DefeasibleEntail}(X,r_X,\psi)=False$ for some $X\in \textbf{PropKB}(s)\setminus\{\KB_*\}$. Then lastly note that $\text{StdptRankEntail}(K,r,\Box_s\psi)=False$ iff $\text{DefeasibleEntail}(X,r_X,\psi)=False$ for some $X\in \textbf{PropKB}(s)\setminus\{\KB_*\}$ (this is clear from lines 7-13 of the algorithm). Thus $\KB\ndentails \Box_s \psi$ iff $\text{StdptRankEntail}(K,r,\Box_s\psi)=False$.
        \item If $\phi=\Diamond_s\psi$ we repeat a similar dual argument: $\KB\dentails \phi$ iff $M_\KB^{\dentails},\pi_X\Vdash \psi$ for some $\pi_X\in \sigma(s)$ iff $\tau_\KB^r(\pi_X)=r_X\Vdash \phi$ for some $X\in \textbf{PropKB}(s)\setminus\{\KB_*\}$. Again, this is equivalent to $X\dentails \psi$ for some $X\in \textbf{PropKB}(s)\setminus\{\KB_*\}$, which by Casini, Meyer and Varzinczak \cite{casini:beyondratclosure} is equivalent to $\text{DefeasibleEntail}(X,r_X,\psi)=True$ for some $X\in \textbf{PropKB}(s)\setminus\{\KB_*\}$. Then note that $\text{StdptRankEntail}(K,r,\Diamond_s\psi)=True$ iff $\text{DefeasibleEntail}(X,r_X,\psi)=True$ for some $X\in \textbf{PropKB}(s)\setminus\{\KB_*\}$ (note lines 14-21 of the algorithm). Thus $\text{StdptRankEntail}(K,r,\Diamond_s\psi)=True$ iff $\KB\dentails\Diamond_s \psi$.
        \item For our inductive step on conjunctions we assume the result holds for each conjunct $\phi_1$ and $\phi_2$ and then we assume $\KB\dentails \phi_1\wedge \phi_2$. In this case, this is equivalent to $M_\KB^{\dentails}\Vdash\phi_1\wedge \phi_2$ which is equivalent to $M_\KB^{\dentails}\Vdash\phi_1$ and $M_\KB^{\dentails}\Vdash \phi_2$. Then we note by inductive hypothesis that this is equivalent to $\text{StdptRankEntail}(K,r,\phi_1)=True$ and $\text{StdptRankEntail}(K,r,\phi_2)=True$ which, by lines 1-6 of the \text{StdptRankEntail} algorithm, is equivalent to $\text{StdptRankEntail}(K,r,\phi_1\wedge \phi_2)=True$.
    \end{itemize}
    
\end{proof}

\textbf{Theorem \ref{theorem:complexities-of-defeasible-entailments}. }\textit{For any propositional defeasible entailment $\dentails_{prop}$, we have that if $\dentails_{prop}$ is computable in a complexity class $C$ such that $\textsc{P}_{\parallel}^\textsc{NP}\subseteq C$, then entailment-checking for $\dentails_{DRSL}$ remains in $C$. In particular:
    \begin{itemize}
        \item Entailment checking for rational closure in DRSL is $\textsc{P}_{\parallel}^\textsc{NP}$-complete.
        \item Entailment checking for lexicographic closure in DRSL is $\textsc{P}^\textsc{NP}$-complete.
    \end{itemize}}

\begin{proof}
    Presume we wish to compute whether $\KB\dentails \phi$, where we assume the size of the input is the size of $\KB$ and $\phi$ (we consider the ranking strategy $r$ as being determinable from $\KB$ and $\phi$). We show this by showing that each query of \texttt{StdptRankEntail} involves at most polynomially many checks to \texttt{DefeasibleEntail}, in the size of $\KB$ and $\phi$.

    Let $n$ be the size of $\KB$ and let $k$ be the length of the query $\phi$. For each conjunct in of the form $\Box_s \psi$ occurring in $\phi$ we complete the if loop at line 7 of \texttt{StdptRankEntail} at most once for each $X\in \textbf{PropKB}_\KB(s)\setminus\{K_*\}$. Each repeat of the loop calls \texttt{DefeasibleEntail} once. Note that as mentioned in previous proofs we have that each set $X\in \textbf{PropKB}_\KB(s)\setminus\{K_*\}$ is no larger than the original knowledge base $\KB$. Similarly $\Box_s\psi$ is no larger than $\phi$. Therefore the input to each call of \texttt{DefeasibleEntail} is no larger than $n+k$.

    Furthermore, we once again recall that for any standpoint symbol $s$ the size of $\textbf{PropKB}_\KB(s)\setminus\{K_*\}$ is no larger than the size of $\KB$. Therefore we only repeat the loop at line 7 at most $n$ times. 

    For similar reasons, for each conjunct of the form $\Diamond_s \psi$ we call \texttt{DefeasibleEntail} at most $n$ times with an input no larger than $n+k$.

    Then note there at most $k$ conjuncts of each form occurring is $\phi$, and so we call \texttt{DefeasibleEntail} at most $n\cdot k$ times with an input no greater than the input of  \texttt{StdptRankEntail}. Furthermore, no input of a call to \texttt{DefeasibleEntail} depends on the outcome of any previous call, and all of the given calls can be made in parallel. 
    Lastly note that $\textbf{PropKB}_\KB(s)$ can be computed in polynomial time in the size of $\KB$ \cite{LMR2024} and can be computed without relying on the output of any call to \texttt{DefeasibleEntail}. That is, \texttt{StdptRankEntail} involves at most polynomially many parallel calls to \texttt{DefeasibleEntail}. Therefore, if the complexity class of this algorithm is in $C$ such that $\textsc{P}_\parallel^\textsc{NP}\subseteq C$, computing \texttt{StdptRankEntail} remains within $C$.

    In particular since \texttt{RationalClosure} in the propositional case is $\textsc{P}_\parallel^\textsc{NP}$-complete, extending it to DRSL is $\textsc{P}_\parallel^\textsc{NP}$-complete. Similarly, since \texttt{LexicographicClosure} in the propositional case is $\textsc{P}^\textsc{NP}$-complete, extending it to DRSL is $\textsc{P}^\textsc{NP}$-complete.
\end{proof}

\end{document}